\documentclass[12pt,a4paper]{article}

\usepackage[margin=1in]{geometry}
\usepackage[doublespacing]{setspace}

\usepackage{amsmath,amsfonts,bm}

\DeclareMathAlphabet{\mathsfit}{\encodingdefault}{\sfdefault}{m}{sl}
\SetMathAlphabet{\mathsfit}{bold}{\encodingdefault}{\sfdefault}{bx}{n}

\usepackage{url}

\usepackage{amsmath}
\usepackage{amssymb}
\usepackage{mathtools}
\usepackage{amsthm}
\usepackage{subfigure}
\usepackage{amsmath}
\usepackage{amssymb}
\usepackage{mathtools}
\usepackage{amsthm}
\usepackage{algorithmic}
\usepackage{caption}
\usepackage{subcaption}
\usepackage{graphicx}

\theoremstyle{plain}
\newtheorem{Theorem}{Theorem}[section]

\newtheorem{Remark}{Remark}[section]

\newtheorem{Definition}{Definition}[section]

\newcommand{\mbb}{\mathbb}

\newcommand{\mcal}{\mathcal}
\newcommand{\op}{\operatorname{op}}
\newcommand{\T}{\top}
\newcommand{\F}{\operatorname{F}}
\newcommand{\diag}{\operatorname{diag}}

\usepackage{wrapfig}
\usepackage{microtype}
\usepackage{graphicx}
\usepackage{subfigure,enumitem}
\usepackage{booktabs}
\usepackage{enumitem}
\usepackage{tcolorbox}
\tcbuselibrary{skins,breakable}
\usepackage[ruled,vlined,linesnumbered]{algorithm2e}
\usepackage{float}

\newtcolorbox{algobox}{ enhanced, colback=white, colframe=black, boxrule=0.5pt, arc=0pt, left=3pt, right=3pt, top=3pt, bottom=3pt, boxsep=2pt }

\def\namedlabel#1#2{\begingroup
    #2%
    \def\@currentlabel{#2}%
    \phantomsection\label{#1}\endgroup
}

\usepackage[toc,page,header]{appendix}
\usepackage{minitoc}
\usepackage{natbib}

\usepackage[colorlinks=true, citecolor=blue, linkcolor=blue, urlcolor=blue]{hyperref}

\title{Graph Memory: Spectral Associative Memory via Dirichlet Energy}

\author{
Zhaoyang Shi\thanks{Corresponding author: 
\href{mailto:zyshi10m@outlook.com}{zyshim@fudan.edu.cn}}\\[0.5em]
Center for Applied Mathematics, Fudan University\\
Shanghai, China
}

\date{}
\begin{document}

\maketitle

\begin{abstract}
Dense associative memories have traditionally focused on storing and retrieving vector-valued patterns. Many modern machine learning problems, however, are naturally graph-structured, requiring memory mechanisms for relational patterns, graph diffusion geometries, community structures, and graph-based inductive biases. We propose a spectral dense associative memory for storage and retrieval of graph data, extending the classical vector-valued memories. Retrieval is performed through a log-sum-exp energy induced by Dirichlet energy with spectral norm distances, producing a softmax-weighted average of the stored Laplacians that remains a valid graph Laplacian. We prove exponential storage capacity and exponentially decaying retrieval error. Beyond graph retrieval, we establish theoretical guarantees for spectral quantities central to graph learning, including eigenvalues, eigenspaces, and diffusion operators. Experiments on synthetic graph data, real-world airline network, protein conformation data and wearable sensor data demonstrate robust graph retrieval while preserving the graph geometry of the data. Our framework provides a new associative memory paradigm for graph-structured data and bridges dense associative memory with modern graph learning and generative AI.
\end{abstract}


\section{Introduction}
Associative memories are a foundational paradigm providing content-addressable memory mechanism that enables robust retrieval of stored information from noisy or incomplete inputs through error-correcting dynamics. From John Hopfield’s seminal paper \citep{hopfield1982neural} to modern dense associative memories (DAMs) \citep{krotov2016dense,hoover2026dense,krotov2025modern}, associative memories have evolved into a powerful framework for storing high-dimensional patterns in distributed representations and retrieving them reliably from partial or corrupted queries. Given a partial, noisy, or corrupted query, DAMs retrieve the corresponding stored memory through energy-based dynamics, while achieving exponentially large storage capacity and revealing deep connections to attention mechanisms in Transformers \citep{ramsauer2020hopfield}. 

Despite these advances, existing DAMs have largely been developed for vector-valued patterns or Euclidean representations, making their extension to data whose essential information is relational rather than coordinate-based far from immediate. Graph representation learning has emerged as one of the central paradigms in modern machine learning, providing a natural representation for interactions, dependencies, and relational inductive biases, underpinning research in graph-based nonparametric regression \citep{shi2024adaptive,shi2025minimax}, self-supervised graph representation learning \citep{you2020graph,zhu2020deep}, graph foundation models \citep{mao2024position,chen2024text}, graph-based information theory \citep{berrett2019efficient,shi2024flexible}, graph neural networks \citep{xu2019powerful,wang2025manifold}, graph transformers \citep{rampavsek2022recipe}, and networked multi-agent reinforcement learning \citep{zhang2018fully,shi2025community}. 

More recently, graph representations have become increasingly important in emerging AI applications, including molecular and protein modeling (AlphaFold 3 in \cite{abramson2024accurate}), graph retrieval-augmented generation in \cite{edge2024local}, and scientific AI such as message passing neural PDE solvers \citep{brandstetter2022message}. These advances underscore that graphs have evolved from a specialized data modality to a fundamental representation for modern AI systems, motivating associative memory architectures that operate directly on graphs. Unlike vector-valued patterns, graphs are non-Euclidean objects: their information is encoded not in individual coordinates, but in pairwise relations among nodes. This raises the central question studied in this work:
\[
\begin{gathered}
\textit{Can associative memories be generalized from vector-valued patterns} \\
\textit{to store and retrieve graph data while preserving their graph geometry?}
\end{gathered}
\]

Only very recently have a few works begun to explore non-Euclidean memory spaces, such as Gaussian distributions \citep{tankala2026dense} and Riemannian manifolds \citep{shi2026intrinsic}. However, our question for graphs differs fundamentally. We address this question through the graph Dirichlet energy. The graph Dirichlet energy provides a natural characterization of relational structures by measuring how graph signals interact with the underlying graph. As the fundamental quantity capturing graph smoothness, diffusion, and spectral representations, it offers a principled geometry for comparing graphs beyond individual edge differences. More importantly, the Dirichlet energy admits an operator representation via graph Laplacians, which play a particularly central role in graph-based learning. For example, in graph convolution, the graph Laplacian serves as the fundamental operator defining graph Fourier bases and graph convolutional filters \citep{defferrard2016convolutional}. In \cite{kreuzer2021rethinking}, the full Laplacian spectrum (eigenvalues and eigenvectors) were used to construct learned positional encodings for a graph Transformer. Using the graph Dirichlet energy and its operator representation, we are able to design a dense associative memory that robustly stores and retrieves the graph memory whose graph signal geometry is closest to a noisy or partial query. In this way, we extend classical
dense associative memories from vector-valued patterns to non-Euclidean relational
objects while preserving the spectral structure relevant for downstream graph-based learning. In this work, we make the following contributions:
\begin{itemize}
    \item We introduce Graph-DAM, a dense associative memory for storage and retrieval of graph-structured data, using an operator-norm LSE energy induced by graph Dirichlet energy.

    \item We establish theoretical guarantees for exponential storage capacity and retrieval, together with recovery guarantees for downstream spectral quantities including eigenvalues, eigenspaces, and diffusion operators.

    \item We demonstrate robust graph retrieval on synthetic SBM graphs, real-world airline,  protein networks and wearable sensor data, where Graph-DAM consistently outperforms Euclidean DAM baselines while preserving the graph geometry.
\end{itemize}
By generalizing dense associative memories to graphs, this work provides a new memory paradigm for graph-structured learning. Such a paradigm enables relational structures such as graphs to be stored and retrieved in their Dirichlet energy geometry, enabling memory-augmented graph reasoning, graph representation learning, and structure-aware inference in generative AI. 

\textbf{Notations.} For a positive integer $N$, let $[N]:=\{1,\ldots,N\}$. Let $\mbb Z_{+}$ be the set of all positive integers. We write $\|\cdot\|$ for the Euclidean norm, $\|\cdot\|_{\op}$ for the operator norm and $\|\cdot\|_{\F}$ for the Frobenius norm. Let $\langle\cdot,\cdot\rangle$ denote the Euclidean inner product. For a matrix $A$, we write $A_{ij}$ for its $(i,j)$-th element. For a set $A$, we use $|A|$ for its cardinality.

\section{Associative Memories on Graph Laplacians}\label{sec:DAMongraph}

\subsection{The log-sum-exponential (LSE) energy}
A central component of associative memories is their energy-based architecture. Classical Hopfield networks employ a quadratic energy function, yielding a linear scaling law $O(d)$ of its storage capacity in dimension $d$. Subsequent works (for example \cite{lucibello2024exponential}) introduce exponential energy functions that enable exponentially large storage capacity. For a query $q\in\mbb{R}^d$ and stored patterns $\{X_i\}_{i=1}^N$, the (Euclidean) log-sum-exponential (LSE) energy in \cite{ramsauer2020hopfield} is
\begin{align}\label{eq:LSE}
    E^{\text{Eu}}(q)=-\frac1\beta\log\left(\sum_{i=1}^N\text{exp}\{\beta\langle q,X_i\rangle\}\right),
\end{align}
where $\beta>0$ is called inverse temperature, and $\langle \cdot,\cdot \rangle$ denotes the Euclidean inner product. As $\beta\rightarrow\infty$, through a Gibbs-type aggregation retrieval with softmax weights, this energy increasingly emphasizes the nearest neighbor stored pattern, thereby producing sharper attraction basins and enabling substantially larger, even exponential, storage capacity. Hence the classical LSE energy operates on vector-valued data, where similarity is measured directly through Euclidean inner product. This is closely related to the attention mechanism in Transformer \citep{ramsauer2020hopfield}.

\subsection{Graph Dirichlet energy and its Laplacian representation}
However, unlike vector-valued Euclidean data, graph representations are inherently relational and operator-valued: their essential information is encoded not in individual coordinates, but in how the graph acts on signals through its connectivity and spectral structure.
Graph Dirichlet energy plays a central role in graph-structured data by encoding the underlying topology and providing a fundamental geometry for graph signal processing, spectral analysis, and diffusion. 

Given a graph $G=(V,E)$ and $n\in \mbb Z_{+}$, the vertex set $V$ consists of $n$ vertices (not necessarily in Euclidean spaces) and $E$ denotes the edge set. Let $W\in\mbb{R}^{n\times n}$ be its weighted adjacency matrix. For a graph signal
\(f\in\mathbb{R}^n\), the Dirichlet energy induced by \(G\) is
\(
    \mathcal{E}_{G}(f)
    =
    \frac12\sum_{i,j=1}^n W_{ij}(f_i-f_j)^2.
\)
Let $D\in\mbb{R}^{n\times n}$ be the degree matrix such that it is diagonal with $D_{ii}=\sum_{j=1}^n W_{ij}$ for $i\in [n]$. The Dirichlet energy admits an operator representation
\(
    \mathcal{E}_{G}(f)
    =
    f^\T L f,
\)
where $L=D-W$ is the graph Laplacian. There are some commonly used graph constructions. In $\epsilon$-neighbor graphs, two nodes $X_i,X_j$ are connected if $d(X_i,X_j)\le \epsilon$ for some $\epsilon>0$ and some dissimilarity metric $d$ (not necessarily in Euclidean spaces). In $k$-nearest neighbor graphs, an edge from $X_i$ to $X_j$ exists if $X_j$ is one $k$-nearest neighbor of $X_i$ for some positive integer $k$. The standard adjacency matrix $W$ is the classical unweighted adjacency matrix $A$ whose elements are one if $X_i,X_j$ are connected and zero otherwise. A general weighted $W_{ij}$ can be defined via a general dissimilarity of $X_i$ and $X_j$.

\subsection{Energy functions on graphs}
Let $\mcal{G}$ be the set of all graphs $G$ of node size $n$. To design a suitable energy function, our motivation is that the core mechanism of associative memories is association: retrieval requires a similarity measure that compares a query with each stored memory. We use the graph Dirichlet energy as the similarity measure. For a query $q\in\mcal{G}$ and stored graphs $\{G_i\}_{i=1}^N$, plugging the Dirichlet energy into the LSE energy in \eqref{eq:LSE}, we obtain
\begin{align*}
    E_f(q)=-\frac1\beta\log\left(\sum_{i=1}^N\text{exp}\{-\beta(\mcal{E}_{G_i}(f)-\mcal{E}_q(f))^2\}\right).
\end{align*}
However, the Dirichlet energy relies on the observed signal $f$. Since
\begin{align}\label{eq:dirichlet-op}
    \sup_{\|f\|_2=1}
    \left|
        \mathcal{E}_{q}(f)-\mathcal{E}_{G_i}(f)
    \right|
    =\sup_{\|f\|_2=1}
    \left|
        f^\top(L_{G_i}-L_q)f
    \right|=\|L_{G_i}-L_q\|_{\mathrm{op}},
\end{align}
we further propose a modified LSE energy due to the Laplacian representation. Throughout this work, we assume that graph memories share a common aligned node set.

As $\beta\rightarrow\infty$, this energy reduces to the minimum squared operator-norm discrepancy between the query and the stored graphs, implementing a soft nearest-neighbor retrieval rule in spectral graph space. Moreover, due to the Laplacian representation, it also has a clear probabilistic interpretation: stored graphs are weighted by their worst-case Dirichlet energy proximity to the query, favoring stored graphs whose Dirichlet energies are similar. Figure \ref{fig:1} shows the energy landscape of five stored stochastic block model (SBM) graphs of two clusters for different $\beta$. As $\beta$ increases from $0.001$ to $10$, the energy landscape develops increasingly sharp and well-separated attraction basins around the stored graph memories. This further supports our graph LSE's capacity of storage. 

\begin{figure}[!t]
\vspace{-5mm}
    \centering
    \includegraphics[width=0.9\textwidth]{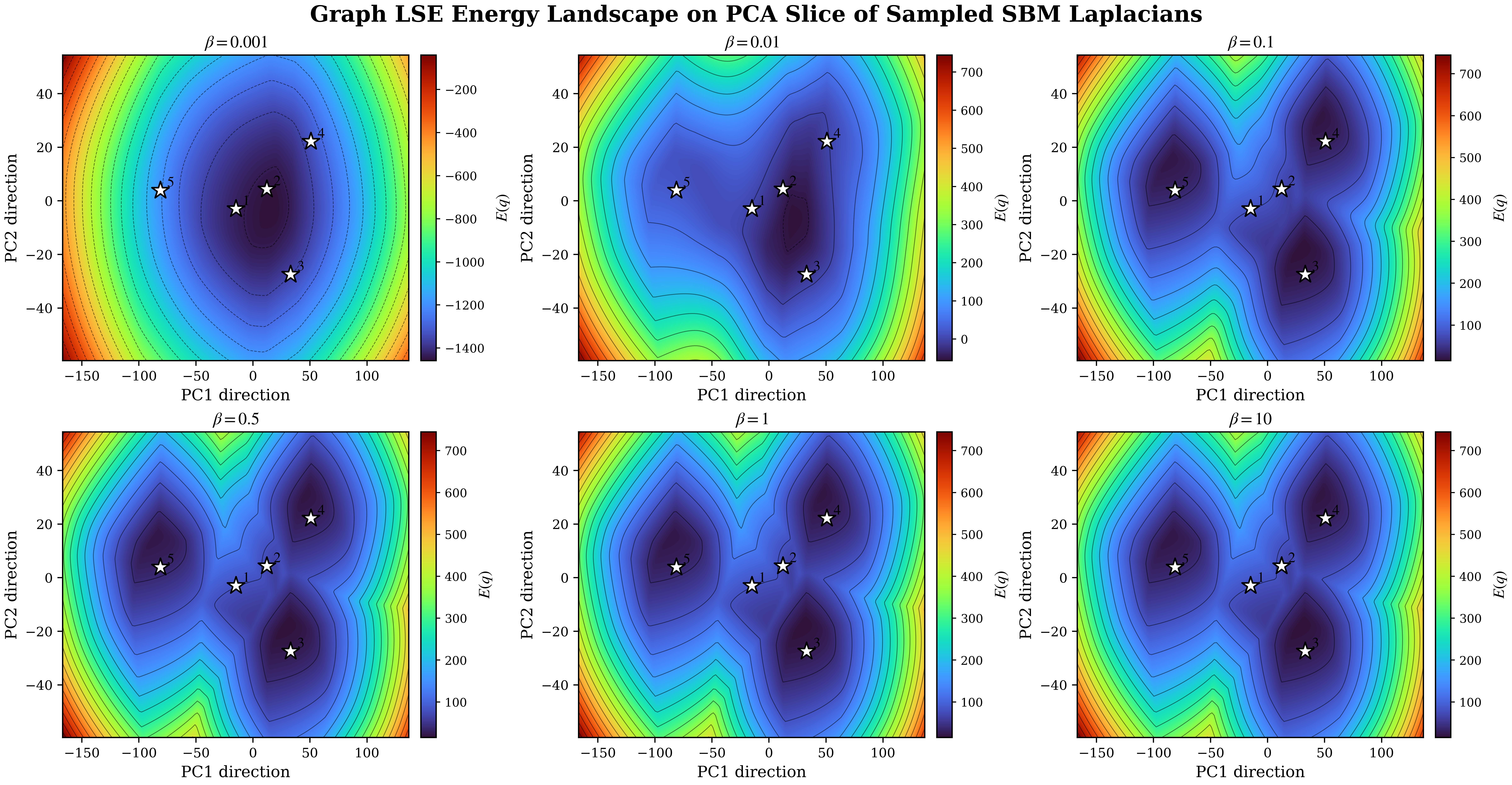}
    \caption{Energy landscape of the proposed graph LSE with $N=5$ stored two-community stochastic block model (SBM) graphs. Each graph has $n=60$ nodes with community sizes $(30,30)$. Since the Laplacian space is high-dimensional, the landscape is visualized by restricting the query graph's Laplacian to the two-dimensional subspace spanned by the first two principal components of the stored graphs' Laplacians. Stars denote the projected locations of the stored graph memories.}
    \label{fig:1}
\end{figure}

For simplicity, let $L_i \equiv L_{G_i}$. We define a Gibbs-type weight and a Gibbs-type retrieval operator $\mcal{T}:\mcal{G}\rightarrow\mcal{L}$:
\begin{align*}
    w_i(q)
        =
        \frac{\exp\{-\beta \|L_q-L_i\|_{\op}^2\}}
        {\sum_{\ell=1}^N \exp\{-\beta \|L_q-L_\ell\|_{\op}^2\}},\quad 
    \mcal{T}(q)=\sum_{i=1}^Nw_i(q)L_i,
\end{align*}
where $\mcal{L}$ is the space of all Laplacians of graphs $(E,V)$ with $|V|=n$. Since $\mcal{T}$ is a convex combination of Laplacians, its image also lies in $\mcal{L}$. Hence retrieval (memory recall) is iterating query $q$ to a fixed point $\mcal{T}(q^*)=q^*$. Due to the Laplacian representation, the retrieval operator $\mcal{T}$ is also a Gibbs barycenter of stored Dirichlet energies rather than an entrywise average of graph representations. After the retrieval, a decoding step is $W=\diag(L)-L$ if necessary to recover the edges of a graph. See the details of the retrieval update in Algorithm \ref{alg:graph-lse-update}.

\subsection{Euclidean Embedding}
\label{sec:euclidean_embedding}

A natural alternative is to embed a graph Laplacian into a Euclidean vector space and apply a classical DAM with energy \eqref{eq:LSE}. Given a graph $G_i$ with Laplacian $L_i \in \mathbb{R}^{n \times n}$, define its Euclidean representation by
\(
X_i := \operatorname{vec}(L_i) \in \mathbb{R}^{n^2},
\)
under a fixed matrix-to-vector ordering. For a query graph $q$ with Laplacian $L_q$, the corresponding Euclidean query is
\(
X_q := \operatorname{vec}(L_q).
\)
Therefore, the Euclidean LSE energy \eqref{eq:LSE} directly applies to $\{X_i\}_{i=1}^N$ and a query $X_q$. After retrieval, one can revert the output vector to a graph Laplacian. This Euclidean dense associative memory (Eu-DAM) treats the Laplacian as an ordinary vector and compares memories through Euclidean inner products. 

However, the Euclidean embedding immediately faces two key limitations. Beyond the requirement of a prescribed vectorization order, it also ignores the graph geometry. In particular, the Euclidean inner product similarity actually coincides with the Frobenius inner product similarity
\(
\left\langle
\operatorname{vec}(L_q),
\operatorname{vec}(L_i)
\right\rangle
=
\langle L_q,L_i\rangle_{\F}
=
\operatorname{tr}(L_qL_i),
\)
where $\langle \cdot,\cdot \rangle_{\F}$ denotes the Frobenius inner product
\(
\langle L_q,L_i\rangle_{\F}
=
\frac{1}{2}
\left(
\|L_q\|_{\F}^2
+
\|L_i\|_{\F}^2
-
\|L_q-L_i\|_{\F}^2
\right).
\)
Consequently, maximizing the Euclidean similarity does not only favor memories close to the query. It also rewards memories with large Frobenius inner product similarity. The Frobenius similarity is not well matched to graph structure: it aggregates entrywise discrepancies across all edges, so many small perturbations caused by diffuse edge noise, degree heterogeneity, or changes in total edge weight can accumulate into a large discrepancy even when the graph’s dominant spectral property is largely unchanged. In contrast, the operator norm similarity admits the variational characterization \eqref{eq:dirichlet-op}. It captures the largest distortion of graph Dirichlet energy over all unit graph signals, and Dirichlet energy directly reflects the graph-signal geometry underlying smoothness, diffusion, and community structure, rather than through accumulated entrywise variation. See more details in the numerical experiments in Section \ref{sec:experiment}.

\begin{algorithm}[!t]

\caption{One-step graph LSE retrieval update (operator $\mcal{T}$)}
\label{alg:graph-lse-update}

\begin{algorithmic}[1]
\REQUIRE Stored graph memories $\{G_i\}_{i=1}^N$, query graph $G^{(t)}$, inverse temperature $\beta>0$
\STATE Compute the query Laplacian $L^{(t)} \gets L_{G^{(t)}}\equiv L(G^{(t)})$
\FOR{$i=1,\ldots,N$}
    \STATE Compute the stored Laplacian $L_i \gets L(G_i)$
    \STATE Compute spectral discrepancy
    \(
    d_i^{(t)} \gets \bigl\|L^{(t)}-L_i\bigr\|_{\mathrm{op}}^2
    \)
\ENDFOR
\STATE Compute Gibbs weights
\(
w_i^{(t)}
\gets
\frac{\exp\{-\beta d_i^{(t)}\}}
{\sum_{j=1}^N \exp\{-\beta d_j^{(t)}\}},
\ i\in [N] .
\)
\STATE Form the one-step retrieved Laplacian 
\(
\widetilde L^{(t+1)}
\gets
\sum_{i=1}^N w_i^{(t)} L_i .
\)
\STATE (optional) Decode the updated graph
\[
G^{(t+1)}
\gets
\mathrm{Decode}\!\left(\widetilde L^{(t+1)}\right)\ \text{via}\ W^{(t+1)}=\diag(\widetilde{L}^{(t+1)})-\widetilde{L}^{(t+1)}.
\]
\RETURN $\widetilde{L}^{(t+1)},G^{(t+1)}$
\end{algorithmic}
\end{algorithm}

\section{Storage and Retrieval Guarantees}\label{sec:thms}
\subsection{Storage Capacity}
We first extend the notion of storage of the classical DAM to our graph memory and show the property of exponential storage capacity.
\begin{Definition}[Graph storage under normalized operator metric]
For two graphs $G,G'\in\mcal{G}$, define the normalized
operator metric
\(
d_n(G,G')
:=
n^{-1/2}\|L_G-L_{G'}\|_{\op}.
\)
For a stored graph \(G_i\), define its basin of radius \(r>0\) by
\(
B_i(r)
:=
\{G:\ d_n(G,G_i)<r\}.
\)
We say that \(G_i\) is stored if there exists a unique
\(G_i^\star\in B_i(r)\) whose Laplacian $L_i^\star$ is the fixed point of the retrieval map \(\mcal{T}\), every initialization
\(G^{(0)}\in B_i(r)\) converges to \(L_i^\star\) via $\mcal{T}$, and the basins are disjoint
\(
B_i(r)\cap B_j(r)=\varnothing\) for \(i\neq j.
\)
\end{Definition}

\begin{Theorem}[Exponential storage capacity]\label{thm:exp-storage}
Let
\[
\mathcal W
=
\left\{
W\in\mathbb R^{n\times n}:
W=W^\top,\quad W_{uu}=0,\quad W_{uv}\in[0,1]
\right\}.
\]
be the set of adjacency matrices. For \(W\in\mathcal W\), let \(L(W)=D(W)-W\) be the graph Laplacian associated with $W$. Suppose we sample $N$ graphs $G_1,\ldots,G_N\in\mcal{G}$ in the following way: Let \(\{W_i\}_{i=1}^N\) be sampled independently by
\(
W_{i,uv}\overset{\mathrm{i.i.d.}}{\sim}\mathrm{Unif}[0,1],\ 1\le u<v\le n,
\)
and set
\(
L_i=L(W_i).
\)
Let $\delta\in(0,1)$ and
\(
N=\lfloor\sqrt{\delta} e^{c_{\text{sep}} n}\rfloor
\)
with some constant $c_{\text{sep}}>0$. There exists $n_0\in \mbb Z_{+}$ such that for all $n\ge n_0$ and \(\beta>\beta_0=48c_{\text{sep}}\), with probability at least $1-\delta$, all \(N\) memories are stored.
\end{Theorem}
\begin{Remark}
    We use a $n^{-1/2}$ normalized metric for $n\times n$ graph Laplacians. It is a well defined metric on the graph space $\mcal G$ since if $L_G=L_{G'}$, their adjacency matrices are equal $W_G=W_{G'}$, resulting in the same graph structure. Moreover, our sampling here considers uniform edges over $[0,1]$ for simplicity. The same argument extends to more general weighted-edge i.i.d. distributions.
\end{Remark}

\begin{Remark}
    The above exponential storage capacity extends the result in the classical DAM in \cite{ramsauer2020hopfield}. One crucial condition is the separation condition. In the classical DAM, they provided a separation condition measured by an inner product $\Delta_i:=x_i^\T x_i-\max_{j\neq i}x_i^\T x_j$ while we use the concentration of measure to show a separation property over graphs under the normalized metric $d_n$. Moreover, our result also improves the classical DAM version. First, \citet[Theorem 3]{ramsauer2020hopfield} only proved the exponential storage capacity for a fixed $\beta =1$ while our result holds for any $\beta\ge \beta_0$. Second, \citet[Theorem 5]{ramsauer2020hopfield} required a basin radius of order $1/\sqrt{\beta N}$. This decreases as more memories are stored. Importantly, when $N$ scales exponentially in the dimension $n$, the basin becomes exponentially small, demonstrating a curse of dimensionality. However, our basin remains a constant size. The robustness of our graph DAM retrieval does not degrade when we store more graph memories.
\end{Remark}

\subsection{Retrieval Guarantees}
Before we present retrieval guarantees, we would like to mention more about our retrieval rule, i.e., the iteration of the retrieval operator $\mcal T$ to its fixed point. As stated in \cite{krotov2025modern}, another retrieval strategy would be to minimize the graph LSE energy by gradient-based optimization while we instead adopt the Gibbs retrieval operator because it possesses several advantages. First, the objective involves the operator norm, which is generally non-differentiable whenever the largest singular value has multiplicity greater than one. Consequently, gradient-based optimization requires subgradient methods or smoothing approximations, introducing additional algorithmic complexity. Second, even after smoothing, gradient descent proceeds through many local updates whose convergence largely depends on a carefully tuned step size. These optimization hyperparameters substantially influence retrieval quality and computational cost. 

\begin{Theorem}[Graph retrieval guarantee]\label{thm:retrievalproperty}
Under the setting of Theorem \ref{thm:exp-storage}, let \(G_1,\ldots,G_N\in\mcal G\) be stored graph memories, there exists a constant $r_0>0$ (the radius of basin) such that for every \(i\in[N]\), the following hold:

\begin{enumerate}
\item If \(G^{(0)}\in B_i(r_0)\), then
\(
G^{(t)}\in B_i(r_0)\) for \(t\ge0,
\)
where \(L_{G^{(t+1)}}=\mcal T(G^{(t)})\).

\item There exists a unique \(G_i^\star\in B_i(r_0)\) whose Laplacian
\(L_i^\star\) is a fixed point of $\mcal T$ such that
\(
\mcal T(G_i^\star)=L_i^\star,
\)
and a constant $\rho\in (0,1)$ (the contraction constant of the map $\mcal T$) such that
\[
d_n(G^{(t)},G_i^\star)
\le
\rho^t d_n(G^{(0)},G_i^\star)
\le
2r_0\rho^t .
\]
Consequently,
\(
G^{(t)}\to G_i^\star
\)
under the metric $d_n$ as $t\rightarrow \infty$.

\item The number of iterations needed to achieve $\varepsilon$-accuracy:
\(
d_n(G^{(t)},G_i^\star)\le \varepsilon
\)
is at most
\[
t
\ge
\frac{\log(2r_0/\varepsilon)}{\log(1/\rho)}.
\]
\end{enumerate}
\end{Theorem}

\begin{Theorem}[One-step retrieval error]\label{thm:retrievalerror}
Under the setting of Theorem~\ref{thm:exp-storage}, there exists a constant $\eta>0$ such that for every stored graph \(G_i\) and every query
\(G\in B_i(r_0)\), the one-step retrieval of the map $\mcal T$ satisfies
\(
d_n(\mcal T(G),G_i)
\le
4\sqrt{\delta n}\,e^{-\eta n}.
\)
Equivalently, in the operator norm,
\(
\|\mcal T(G)-L_i\|_{\op}
\le
4\sqrt{\delta}\,n\,e^{-\eta n}.
\)
Thus the one-step retrieval error decays exponentially in the dimension \(n\).
\end{Theorem}

\subsection{Spectral Property Guarantees}
In this section, we establish theoretical guarantees for recovering spectral properties of the stored graph memories.
\begin{Theorem}[Spectral recovery after \(t\) retrieval steps]\label{thm:spectralerror}
Under Theorem~\ref{thm:retrievalproperty}, for the initialization $G^{(0)}\in B_i(r_0)$, there exists $\tilde{\rho}_n\asymp \beta n^2\sqrt{\delta}
    \exp\{-(\beta\gamma_0-c_{\rm sep})n\}=o(1)$ as $n\rightarrow \infty$ such that for every \(t\ge0\):
\begin{enumerate}
\item(Eigenvalue error). \(
\|L_{G^{(t)}}-L_i^\star\|_{\op}
\le
2r_0\sqrt n\tilde{\rho}_n^t.
\)
Consequently, for every \(k\in[n]\), the $k$-th smallest eigenvalues of $L_{G^{(t)}}$ and $L_i^*$ satisfy
\[
|\lambda_k(L_{G^{(t)}})-\lambda_k(L_i^\star)|
\le
2r_0\sqrt n\tilde{\rho}_n^t.
\]
\item(Eigenvector error). Fix \(K\ge1\), and let
\(
U_K^\star
=
[u_1^\star,\ldots,u_K^\star]
\)
be the matrix of the first \(K\) eigenvectors of \(L_i^\star\). Let
\(
U_K^{(t)}
=
[u_1^{(t)},\ldots,u_K^{(t)}]
\)
be the matrix of the first \(K\) eigenvectors of \(L^{(t)}\). Suppose the eigengap of $L_i^*$ satisfies
\(
\mu_K
=
\lambda_{K+1}^\star-\lambda_K^\star>0.
\)
Then
\[
\|\sin\Theta(\widehat U_K^{(t)},U_K^\star)\|_{\op}
\le
4r_0\mu_K^{-1}\sqrt n\tilde{\rho}_n^t.
\]
\item(Diffusion error). For every diffusion time \(\tau>0\),
\(
\|e^{-\tau L_{G^{(t)}}}-e^{-\tau L_i^\star}\|_{\op}
\le
2\tau r_0\sqrt n\tilde{\rho}_n^t.
\)
\end{enumerate}
\end{Theorem}

\section{Numerical Experiments}\label{sec:experiment}
For all experiments in this section, for each $\beta=\beta_n$, the retrieval accuracy $\operatorname{Acc}(\beta)$ is measured as the proportion of corrupted queries that are correctly retrieved to their corresponding stored memories. $\beta_n$ ranges over $25$ logarithmically spaced values from $10^{-1}$ to $10^3$. We repeat the full experiment over $5$ independent runs and report the average retrieval accuracy with standard-error bars. The code is available at \url{https://github.com/anonymous10m/Graph-DAM}. Additional experiments on human activity sensor data and wearable sensor data are included in Section \ref{sec:app_experiment} in the appendix.

\subsection{Synthetic SBM Data}
\label{sec:exp_sbm}
We first consider synthetic stochastic block model (SBM) graph data, where the underlying graph geometry lies in its community/cluster structure. We generate $N=50$ graph memories. Each graph has $n=120$ nodes partitioned into $K=4$ equally sized communities with community labels $z_u\in[K]$ for $u\in [n]$. For memory $i$, its SBM is specified by a symmetric block-probability matrix
\(
B_i=(p_{k\ell}^{(i)})\in [0,1]^{4\times 4}.
\)
For the diagonal within community probability $p_{kk}^{(i)}$, we let
\[
\epsilon^{(i)}_1,\ldots,\epsilon^{(i)}_4
\overset{\mathrm{i.i.d.}}{\sim}\mathcal N(0,0.065^2),\quad 
\delta^{(i)}_k
=
\epsilon^{(i)}_k
-
\frac14\sum_{r=1}^4\epsilon^{(i)}_r,
\]
and set
\(
p^{(i)}_{kk}=0.35+\delta^{(i)}_k.
\)
Similarly, for the off-diagonal across community probability $p_{k\ell}^{(i)}$, we let 
\[
\eta^{(i)}_{k\ell}
\overset{\mathrm{i.i.d.}}{\sim}
\mathcal N(0,0.035^2),\quad 
p^{(i)}_{k\ell}
=
p^{(i)}_{\ell k}
=
0.07+\xi^{(i)}_{k\ell}.
\]
For separation, a candidate $B_i$ is retained only if 
$p^{(i)}_{kk}\in[0.21,0.49]$, 
$p^{(i)}_{k\ell}\in[0.025,0.14]$, and
\(
\min_{j<i}\|B_i-B_j\|_{\mathrm{op}}\ge 0.052.
\)
We repeat this procedure until $50$ block matrices are obtained. Given $B_i$, we sample one undirected SBM graph according to
\(
A_{i,uv}\sim
\operatorname{Bernoulli}\!\left(B_{i,z_u z_v}\right),
\ 1\le u<v\le n,
\)
independently across unordered node pairs, and set
$A_{i,vu}=A_{i,uv}$ and $A_{i,uu}=0$.
The corresponding Laplacian is
\(
L_i=D_i-A_i.
\)
For each stored memory $i$, we uniformly remove $20\%$ of the existing edges of the same realized graph $A_i$ and recompute the query Laplacian $L_i^{(q)}$. We compare our method Graph-DAM with the classical Euclidean DAM, Eu-DAM, in Section \ref{sec:euclidean_embedding}. 

Figure~\ref{fig:sbm_results}(a) shows a clear advantage of Graph-DAM. As $\beta_n$ increases, Graph-DAM demonstrates a sharp transition and reaches $100\%$ retrieval accuracy, whereas Eu-DAM improves more gradually and saturates at approximately $86\%$. Figures~\ref{fig:sbm_results}(b)-(d) further illustrate the recovery of the underlying graph geometry. The Fiedler vector retrieved by Graph-DAM closely matches that of the target, preserving both its community structure ($\lambda_2=4.374$). In contrast, Eu-DAM retrieves a graph with a substantially different Fiedler vector and eigenvalue ($\lambda_2=6.160$). This illustrates that accurate graph retrieval under the Dirichlet-energy-induced Graph-DAM preserves downstream spectral structure.

\begin{figure}[!t]
\vspace{-5mm}
    \centering

    \begin{minipage}[c]{0.45\textwidth}
        \centering
        \includegraphics[width=\linewidth]{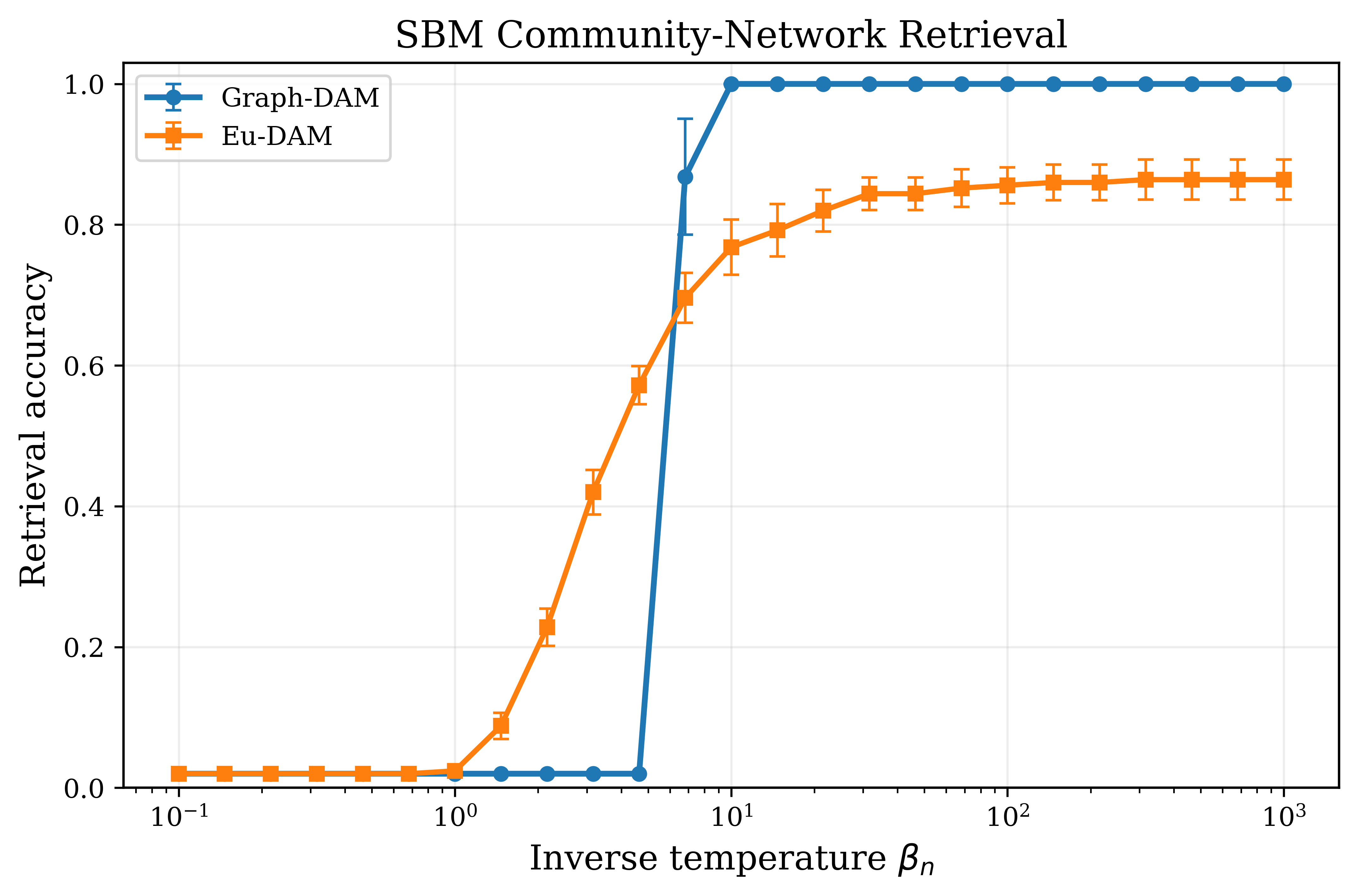}
        
        \small (a) Retrieval accuracy
    \end{minipage}
    \hfill
    \begin{minipage}[c]{0.52\textwidth}
        \centering

        \begin{minipage}[c]{0.52\linewidth}
            \centering
            \includegraphics[width=\linewidth]{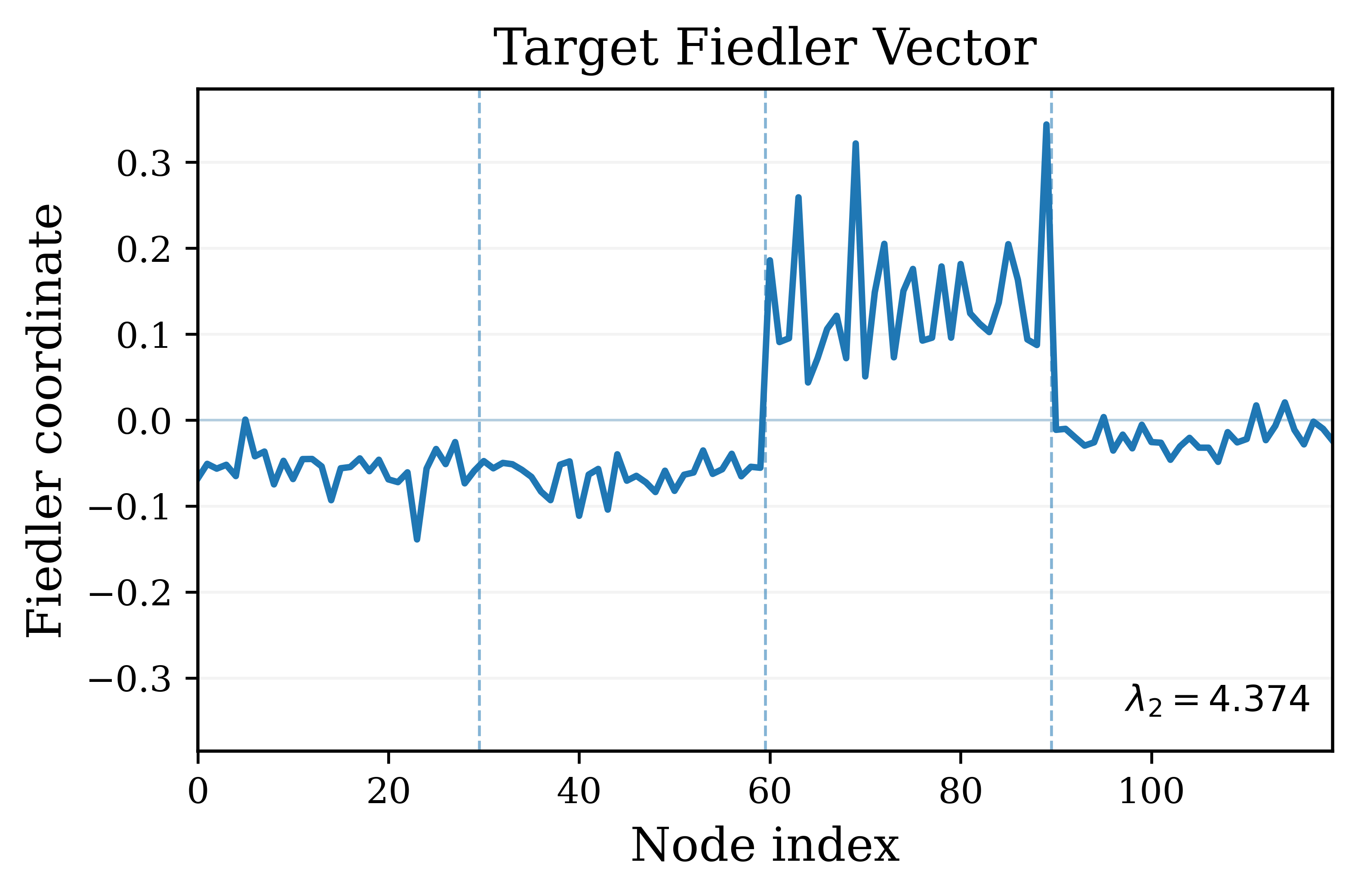}
            
            \small (b) Target memory
        \end{minipage}

        \vspace{2mm}

        \begin{minipage}[c]{0.48\linewidth}
            \centering
            \includegraphics[width=\linewidth]{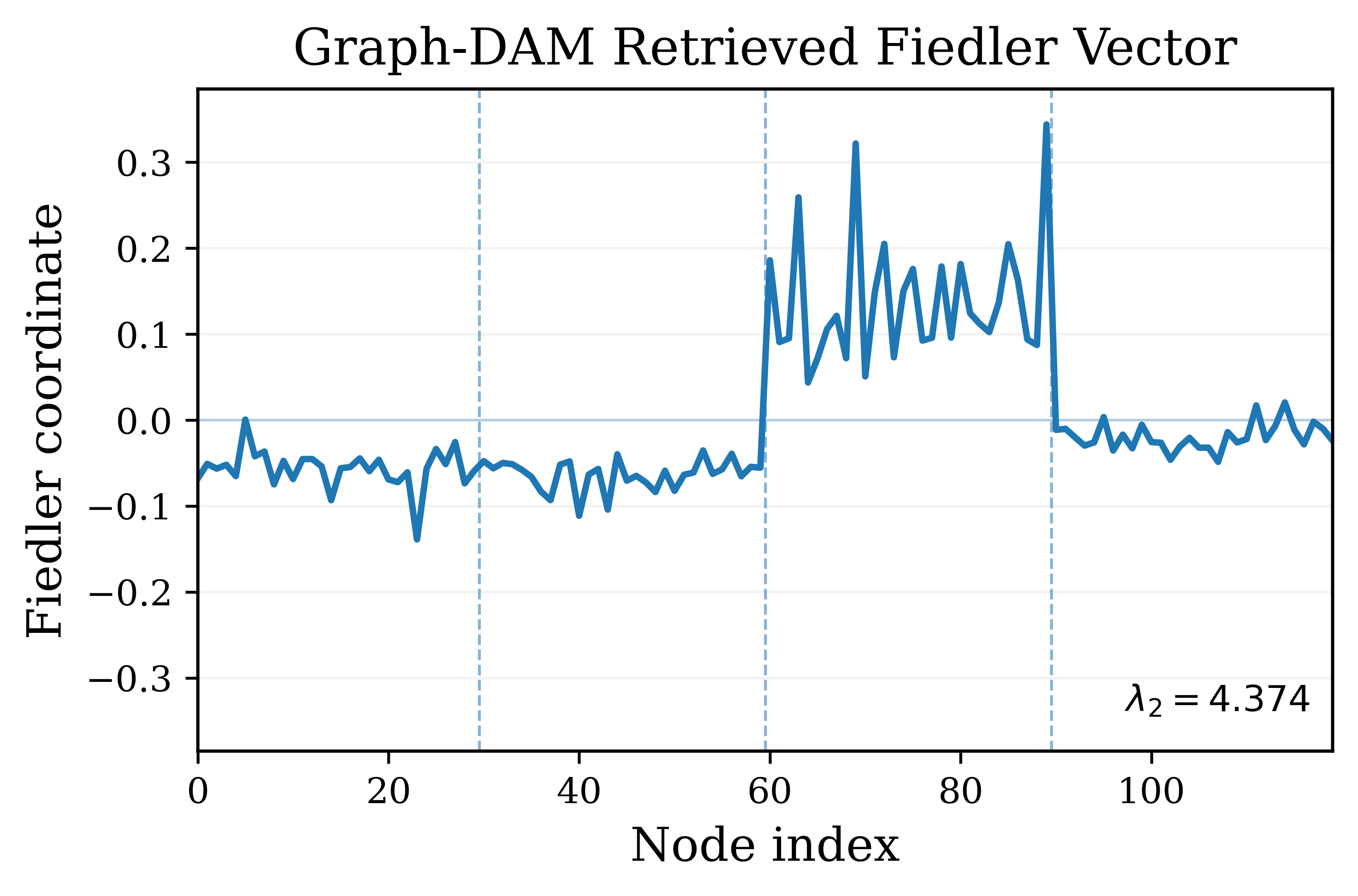}
            
            \small (c) Graph-DAM
        \end{minipage}
        \hfill
        \begin{minipage}[c]{0.48\linewidth}
            \centering
            \includegraphics[width=\linewidth]{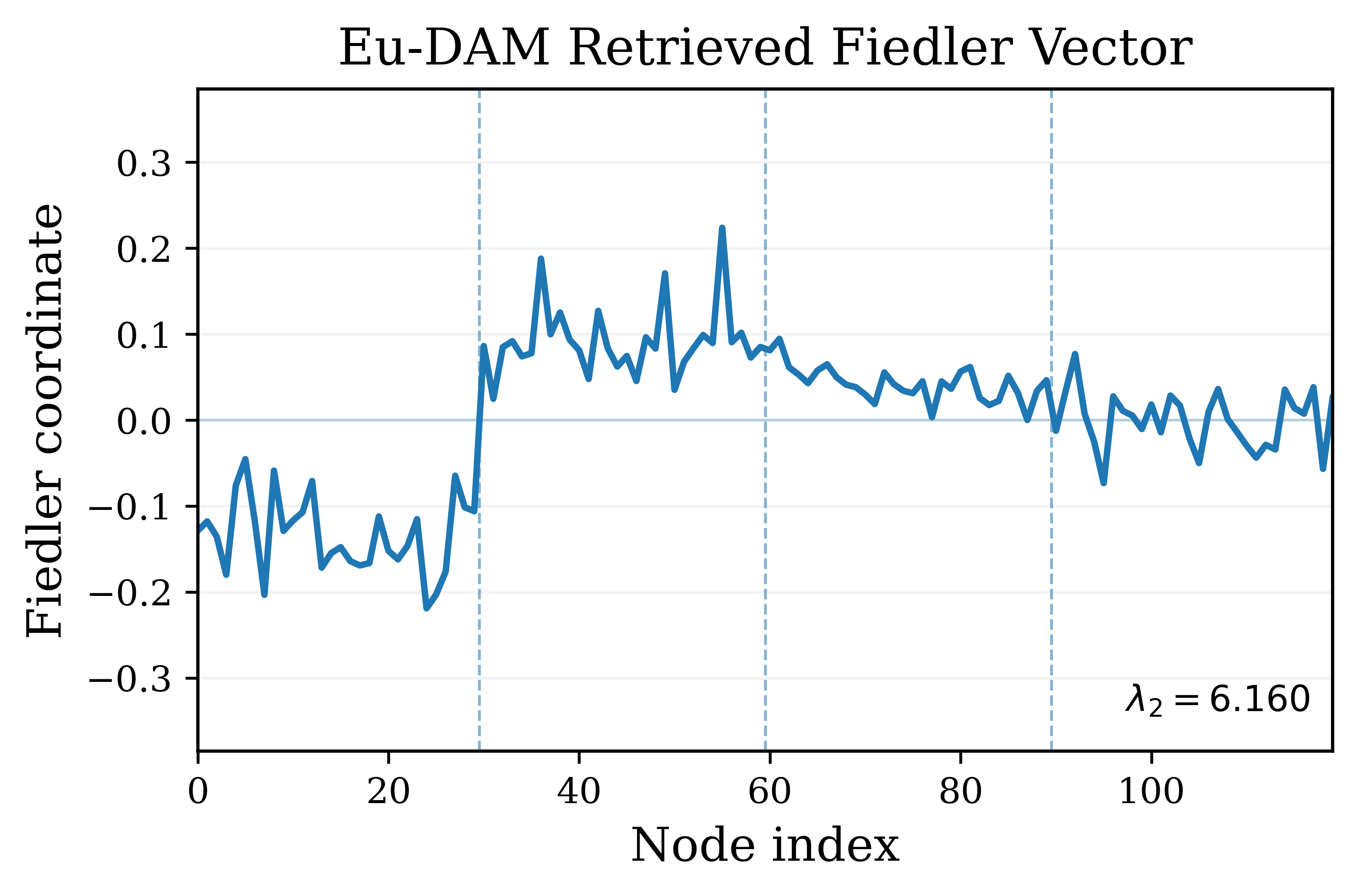}
            
            \small (d) Eu-DAM
        \end{minipage}
    \end{minipage}

    \caption{SBM data retrieval performance.
    (a) Retrieval accuracy of Graph-DAM and Eu-DAM versus inverse temperature
    $\beta_n$.
    (b)-(d) Fiedler vectors of the target memory, Graph-DAM retrieval,
    and Eu-DAM retrieval. Dashed lines indicate the four ground-truth communities.
    }
    \label{fig:sbm_results}
\end{figure}

\subsection{Real-World Airline Network Data}
\label{sec:airline}

We next evaluate graph retrieval on the \texttt{nycflights13} dataset, available at \url{https://nycflights13.tidyverse.org/}, where each memory represents the flight network of an airline carrier. We split the data temporally into January-June for constructing the stored memories and July-December for constructing queries. We retain $N=11$ major carriers (such as American Airlines (AA), Delta Air Lines (DL), and United Airlines (UA)) with at least $1200$ recorded flights in each half-year and restrict all networks to a common node set consisting of the $n=45$ most frequently occurring airports (for example, John F.\ Kennedy International Airport (JFK) and Los Angeles International Airport (LAX)). For carrier $i$, its stored adjacency matrix $A_i\in\mathbb{R}^{45\times45}$ is constructed as follows: $A_{i,uv}$ records the number of observed flights between the two airports $u$ and $v$, and the corresponding Laplacian is
\(
L_i=D_i-A_i.
\)
For each carrier $i$, we uniformly remove $20\%$ of the flight records and construct the query adjacency matrix $A_i^{(q)}$ from the remaining $80\%$, followed by its Laplacian $L_i^{(q)}=D_i^{(q)}-A_i^{(q)}$. In addition to Graph-DAM and Eu-DAM, we consider Raw Eu-DAM, which applies the Euclidean DAM directly to the original flight records without constructing a graph. 

Figure \ref{fig:airline_retrieval}(a) suggests Graph-DAM reaches high retrieval accuracy as $\beta_n$ increases, whereas Eu-DAM and Raw Eu-DAM plateau at lower accuracies. In this dataset, each carrier is represented by a weighted airport network with its own characteristic hub locations, route concentration, and connectivity pattern. For example, the true EV memory in Figure \ref{fig:airline_retrieval}(b) has a New York–centered hub-and-spoke structure, with most of its high-weight connections radiating from the New York airports to a set of regional destinations. This graph property is preserved by our Graph-DAM, retrieving the true memory in Figure \ref{fig:airline_retrieval}(c) while Eu-DAM and Raw Eu-DAM retrieved wrong memories (UA and US) in Figures \ref{fig:airline_retrieval}(d)-(e).

\begin{figure}[!t]
\vspace{-3mm}
    \centering

    \begin{minipage}[c]{0.45\textwidth}
        \centering
        \includegraphics[width=\linewidth]{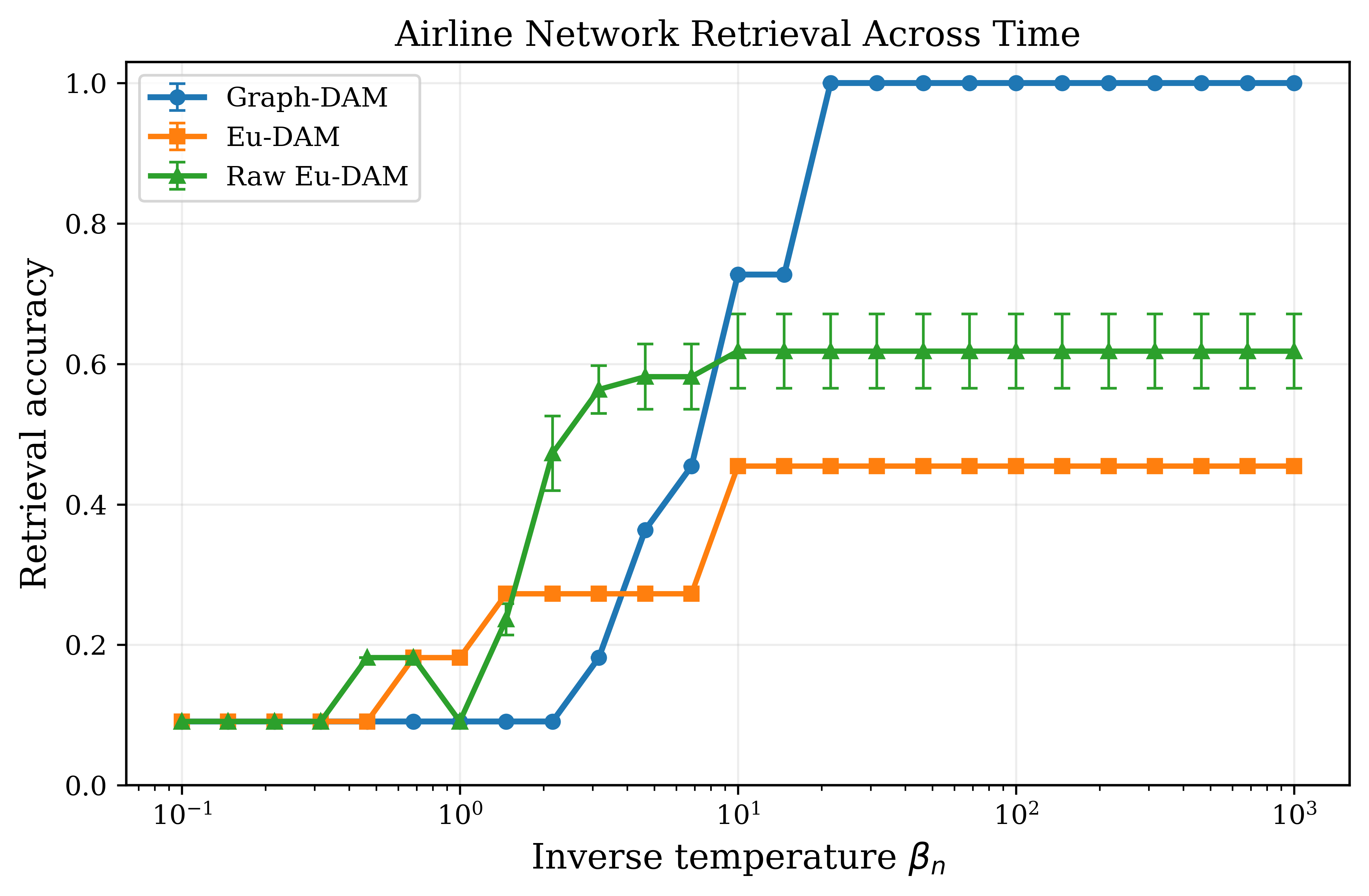}
        
        \small (a) Retrieval accuracy
    \end{minipage}
    \hfill
    \begin{minipage}[c]{0.52\textwidth}
        \centering

        \begin{minipage}[c]{0.49\linewidth}
            \centering
            \includegraphics[width=\linewidth]{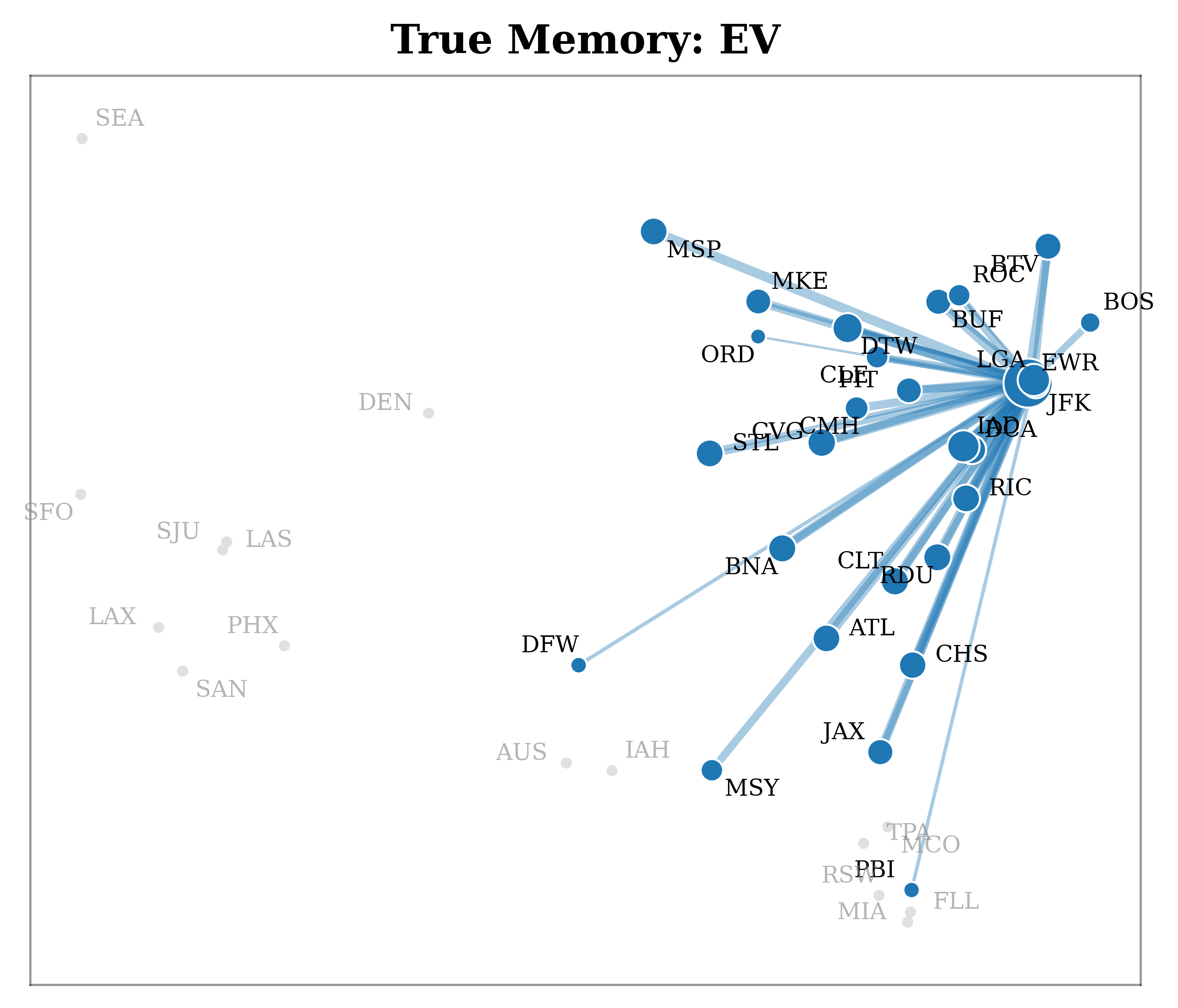}
            
            \small (b) Target memory
        \end{minipage}
        \hfill
        \begin{minipage}[c]{0.49\linewidth}
            \centering
            \includegraphics[width=\linewidth]{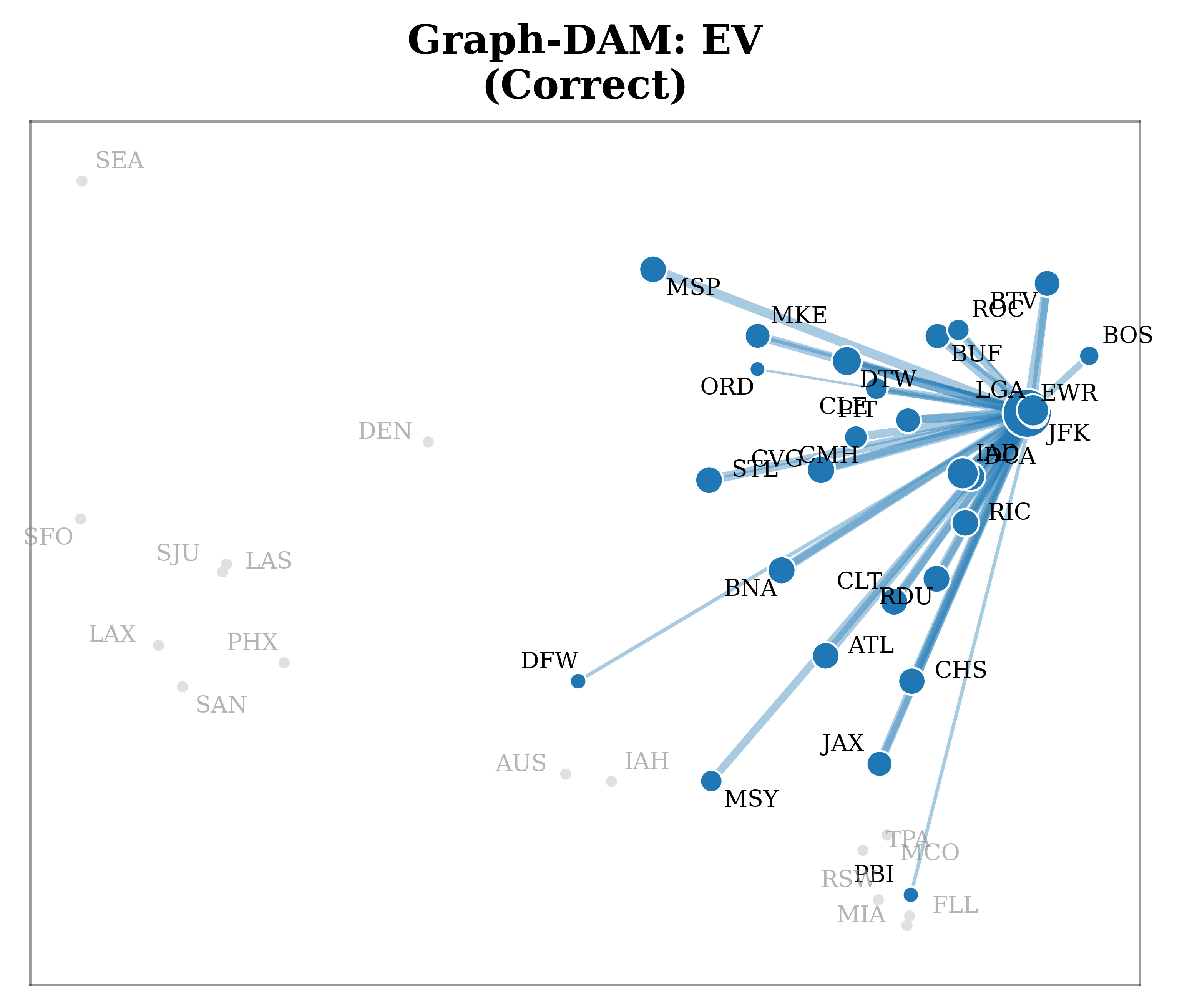}
            
            \small (c) Graph-DAM
        \end{minipage}

        \vspace{1mm}

        \begin{minipage}[c]{0.49\linewidth}
            \centering
            \includegraphics[width=\linewidth]{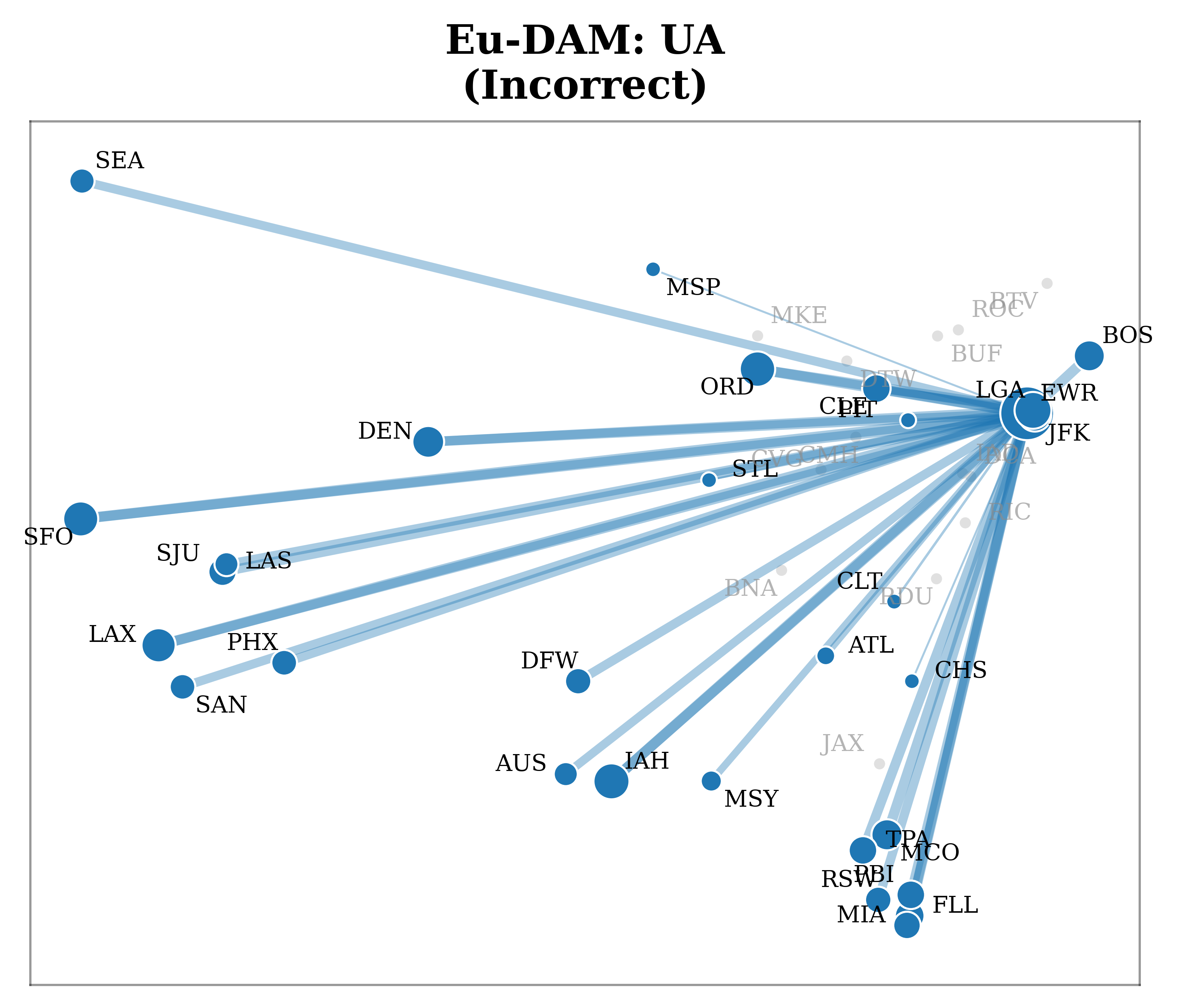}
            
            \small (d) Eu-DAM
        \end{minipage}
        \hfill
        \begin{minipage}[c]{0.49\linewidth}
            \centering
            \includegraphics[width=\linewidth]{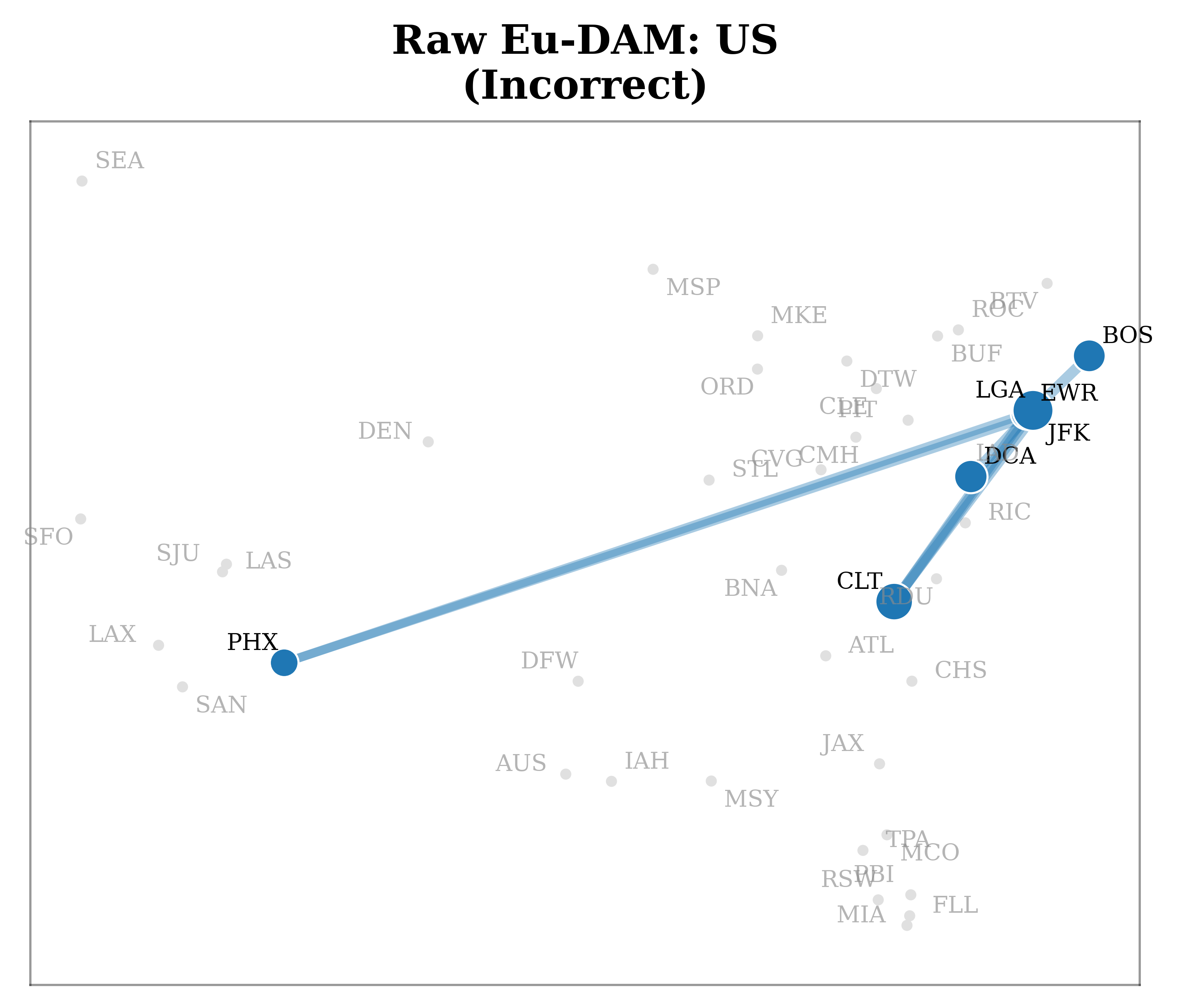}
            
            \small (e) Raw Eu-DAM
        \end{minipage}
    \end{minipage}

    \caption{
    Airline-network retrieval performance.
    (a) Retrieval accuracy of Graph-DAM, Eu-DAM, and Raw Eu-DAM versus inverse temperature $\beta_n$.
    (b) Target network.
    (c) Graph-DAM correctly retrieves the target memory.
    (d)-(e) Eu-DAM and Raw Eu-DAM retrieve an incorrect memory.
    }
    \label{fig:airline_retrieval}
\end{figure}

\subsection{Real-World Protein Conformation Data}
\label{sec:protein}

We further evaluate Graph-DAM on experimental protein structures from the RCSB Protein Data Bank (PDB), available at \url{https://www.rcsb.org}. We consider proteins having at least $8$ available NMR conformations (3D structural model of a protein) and a common chain containing at least $n=80$ residues (one amino-acid unit in a protein chain) with observed C$\alpha$ coordinates (coordinate of a residue with alpha-carbon). For each protein, we retain the first $80$ common residues and center each conformation by its mean coordinate. For protein memory $i$ and its conformation $s$, we construct a weighted residue contact network: Residues (nodes) $u$ and $v$ are connected when their C$\alpha$ distance $d_{uv}$ is at most $8\,\text{\AA}$, with edge weight
\(
A_{uv}
=
\exp\{-(d_{uv}/8)^2\},
\)
and consecutive residues are retained as edges with unit weight. The corresponding Laplacian of $s$-th conformation is $L(X_{is})=D(X_{is})-A(X_{is})$. The stored graph memory is the averaged Laplacian $L_i$ over the first $6$ experimental conformations ($s\in [6]$). For each protein $i$, we construct the query from an experimental NMR conformation that is not used in forming $L_i$. We uniformly remove $20\%$ of the existing contact edges while preserving the sequential backbone connections. The resulting adjacency matrix $A_i^{(q)}$ defines the query Laplacian
\(
L_i^{(q)}=D_i^{(q)}-A_i^{(q)}.
\)
In addition to Graph-DAM and Eu-DAM, we consider Raw Eu-DAM, which operates directly on the C$\alpha$ coordinates of conformations.

Figure~\ref{fig:protein_retrieval}(a) suggests that Graph-DAM reaches high retrieval accuracy as $\beta_n$ increases, whereas Eu-DAM remains substantially lower and Raw Eu-DAM plateaus at a lower accuracy. The poor performance of Eu-DAM confirms our motivation to use Graph-DAM: Small conformational changes of a protein can alter many thresholded contacts, which accumulate under the Frobenius inner product, whereas the Dirichlet induced graph energy remains robust. In this dataset, each protein is not merely a 3D shape but exhibits its characteristic three-dimensional folding, represented by a residue contact network and long-range residue interactions. For example, the target protein 2MBL:A in Figure~\ref{fig:protein_retrieval}(b) exhibits a distinctive folded conformation. This structural organization is preserved by Graph-DAM, which correctly retrieves 2MBL:A in Figure~\ref{fig:protein_retrieval}(c), while Eu-DAM and Raw Eu-DAM retrieve different protein memories (2KUB:A and 2LTM:A) in Figures~\ref{fig:protein_retrieval}(d)-(e).

\begin{figure}[!t]
\vspace{-5mm}
    \centering

    \begin{minipage}[c]{0.45\textwidth}
        \centering
        \includegraphics[width=\linewidth]{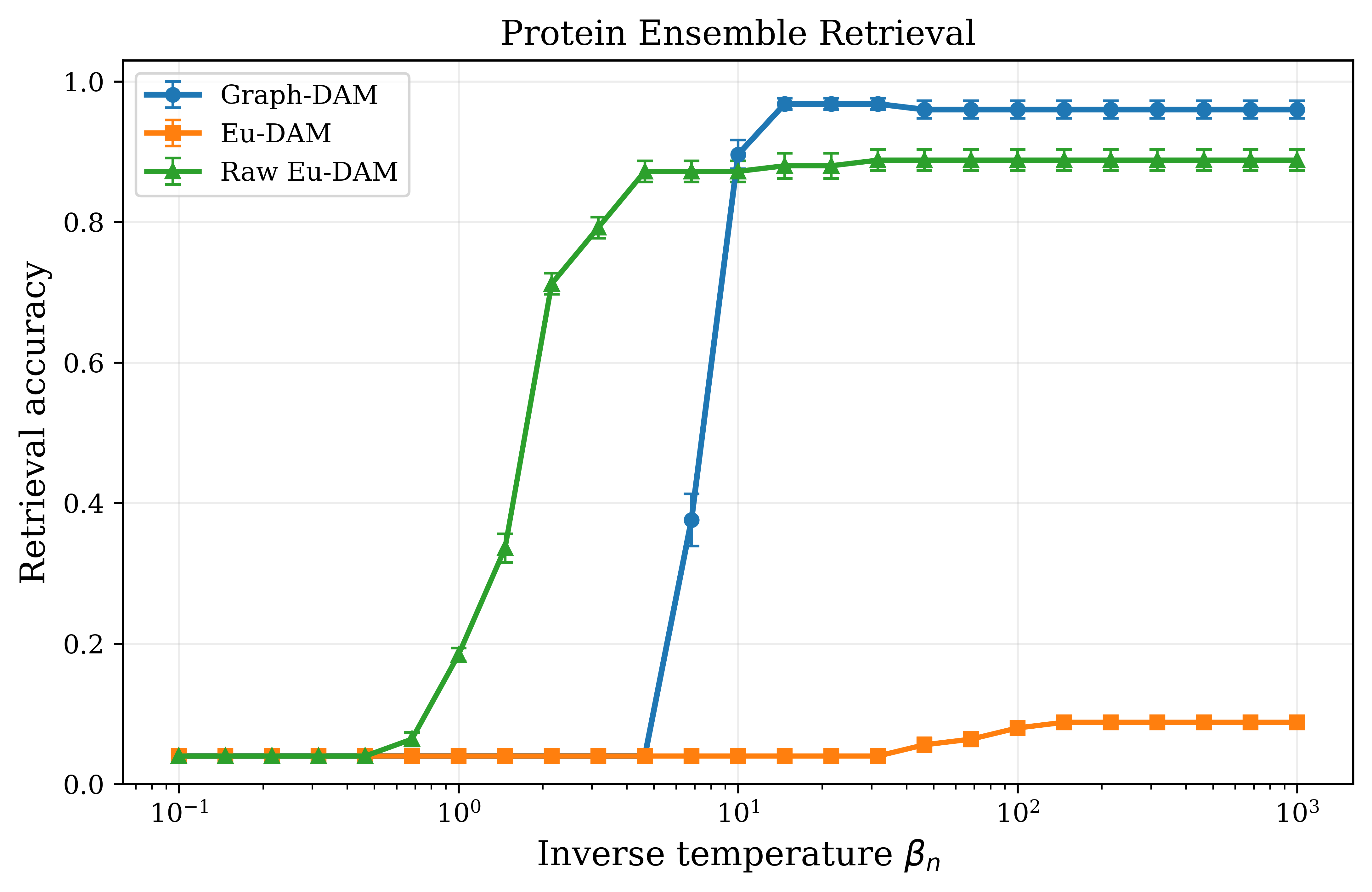}
        
        \small (a) Retrieval accuracy
    \end{minipage}
    \hfill
    \begin{minipage}[c]{0.52\textwidth}
        \centering

        \begin{minipage}[c]{0.49\linewidth}
            \centering
            \includegraphics[width=\linewidth]{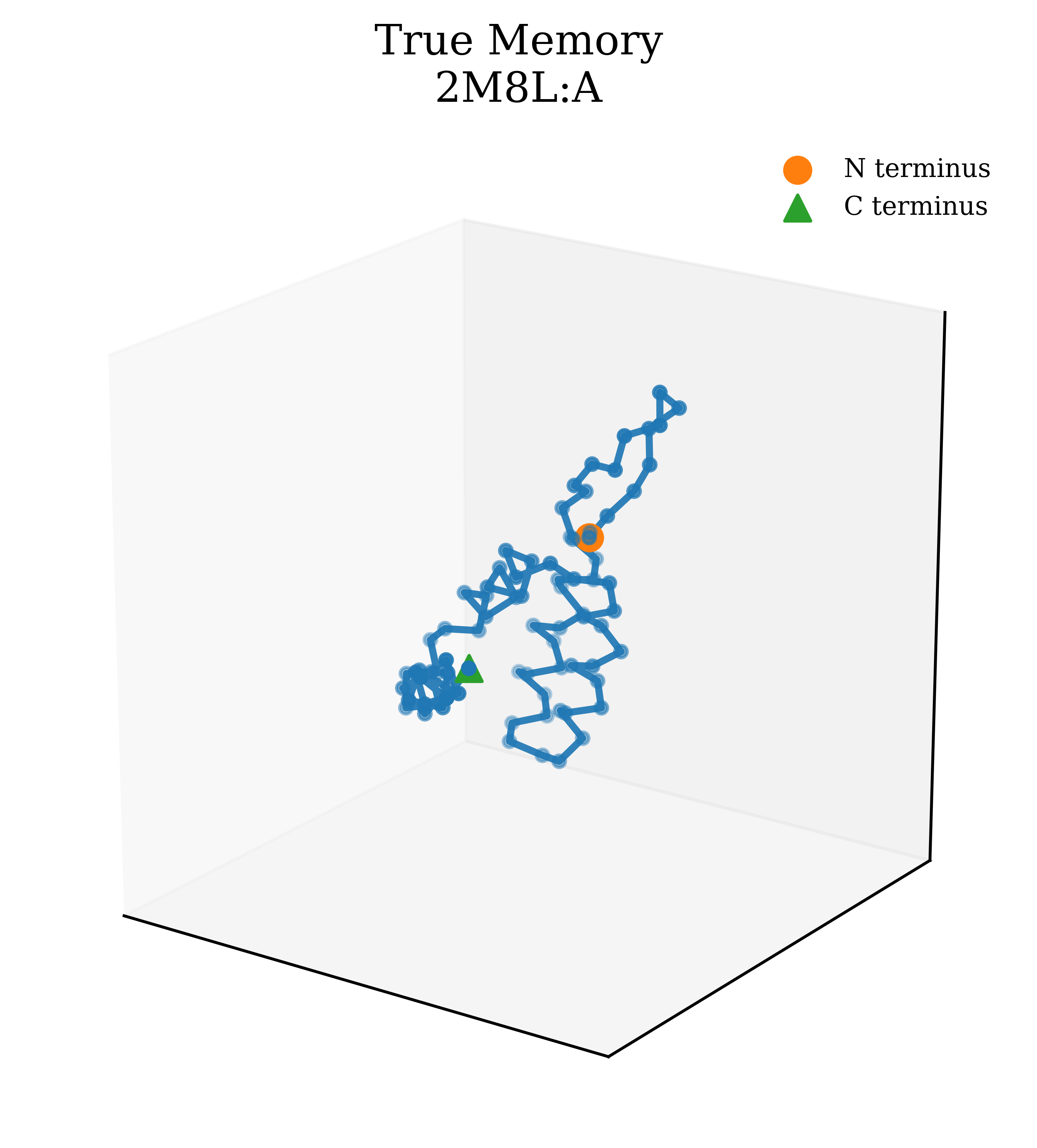}
            
            \small (b) Target memory
        \end{minipage}
        \hfill
        \begin{minipage}[c]{0.49\linewidth}
            \centering
            \includegraphics[width=\linewidth]{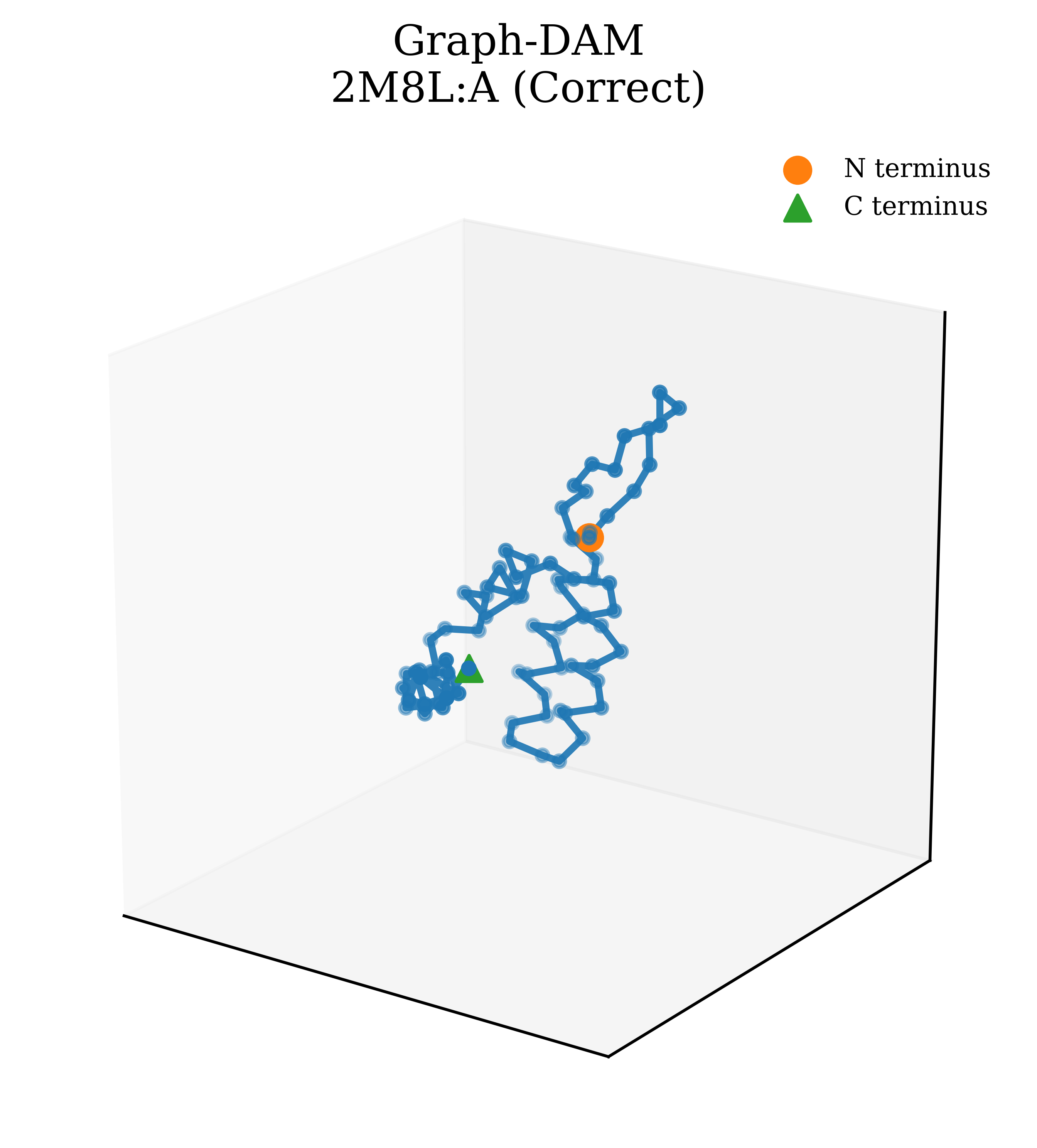}
            
            \small (c) Graph-DAM
        \end{minipage}

        \vspace{1mm}

        \begin{minipage}[c]{0.49\linewidth}
            \centering
            \includegraphics[width=\linewidth]{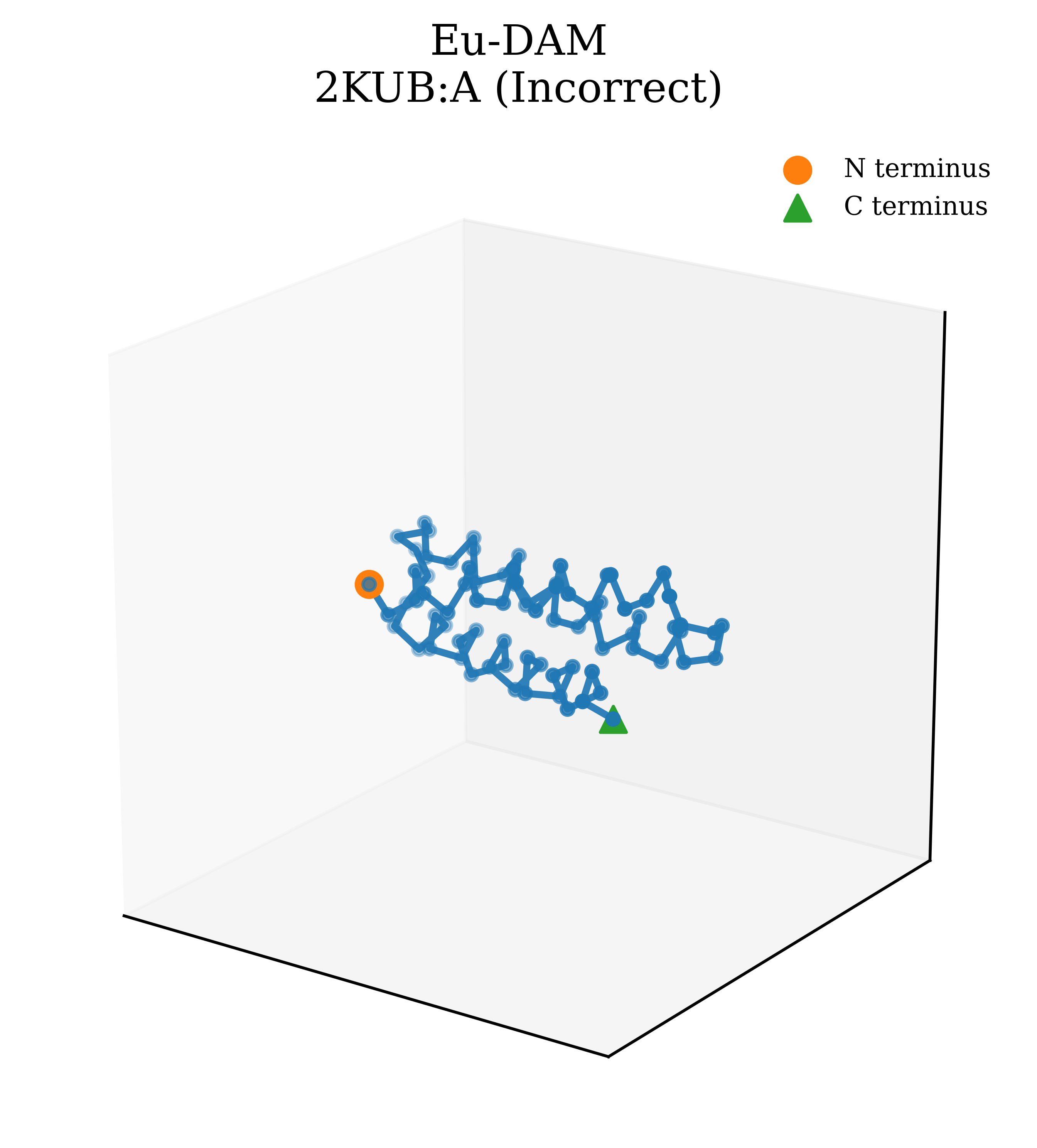}
            
            \small (d) Eu-DAM
        \end{minipage}
        \hfill
        \begin{minipage}[c]{0.49\linewidth}
            \centering
            \includegraphics[width=\linewidth]{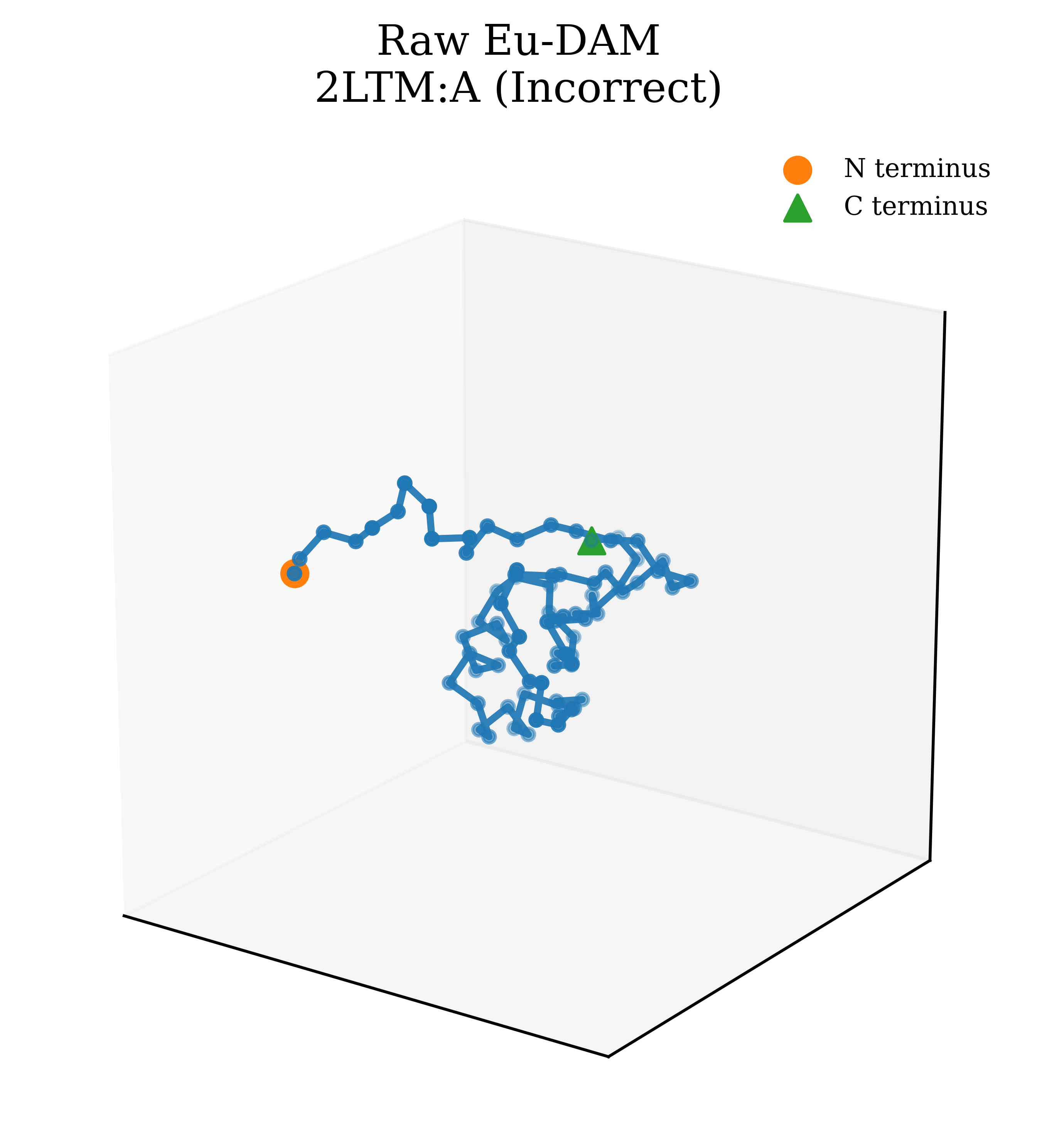}
            
            \small (e) Raw Eu-DAM
        \end{minipage}
    \end{minipage}

    \caption{Protein conformation retrieval.
(a) Retrieval accuracy versus inverse temperature $\beta_n$.
(b) Target protein conformation.
(c) Graph-DAM correctly retrieves the target memory.
(d)--(e) Eu-DAM and Raw Eu-DAM retrieve incorrect protein memories.
}
    \label{fig:protein_retrieval}
\end{figure}

\section{Conclusion and Limitations}

We introduced Graph-DAM, a dense associative memory framework that retrieves graph-structured memories motivated by the graph Dirichlet energy. Our theory establishes guarantees for storage capacity, retrieval, and spectral recovery, while experiments on synthetic and real-world data demonstrate robust retrieval under partial observations and advantages over Euclidean associative memories. Our current framework assumes a common, aligned node space across memories and focuses on undirected weighted graphs represented by their Laplacians. This excludes settings with unknown node correspondence, varying graph sizes, rich node or edge attributes, and directed or dynamically evolving graphs. Extending Graph-DAM to these more general graph structures, while retaining theoretical retrieval guarantees, is an important direction for future work.

\subsection*{AI use statement}
In this work, we have not used generative AI tools to generate synthetic data sets, help develop theoretical models or conceptual frameworks, formulate mathematical claims, provide critical ingredients for proving mathematical claims, assist in the writing of proofs, propose or refine hypotheses, design or provide feedback on research  methodology or experiments, implement methods, clean and reformat dataset, interpret results,
and assisting with translation and support qualitative and thematic data analysis are not applicable to this work.
We used generative AI tools to summarize or analyse existing literature, identify relevant literature, for sourcing/searching for information, for language editing, and for stylistic polishing. We reviewed all AI-assisted work and independently verified the identified literature, datasets, and formatting suggestions against the original sources and official submission guidelines. We take responsibility for the final content of this work,
including text, claims or artifacts produced with the aid of generative AI.

\subsection*{Reproducibility Statement}

We have taken several steps to facilitate the reproducibility of our results. The proposed Graph-DAM framework and its retrieval procedure are fully specified in Section \ref{sec:DAMongraph}, including the graph LSE energy and the one-step retrieval algorithm (Algorithm \ref{alg:graph-lse-update}). The assumptions and statements of all theoretical guarantees are provided in Section \ref{sec:thms}, with complete proofs and supporting technical results included in the appendix. The experimental settings, data construction, evaluation metrics, baselines, and hyperparameter ranges for the synthetic SBM, airline network, and protein conformation experiments are described in Section \ref{sec:experiment}. In addition, we provide the source code implementing the proposed method and experiments at \url{https://github.com/anonymous10m/Graph-DAM}.

\clearpage
\appendix

\section{Related Work}
\label{app:related_work}

\paragraph{Classical and dense associative memories.}
Associative memories provide content-addressable mechanisms for storing patterns and recovering them from partial or corrupted observations. The classical Hopfield network \citep{hopfield1982neural} introduced an energy-based formulation in which stored binary patterns correspond to attractors of the network dynamics. While classical Hopfield networks have storage capacity that scales linearly with the ambient dimension, subsequent dense associative memory (DAM) models substantially increase storage capacity by employing higher-order or nonlinear energy functions. In particular, \citet{krotov2016dense} introduced dense associative memories based on higher-order interactions, demonstrating that the choice of energy function can dramatically increase the number of reliably stored patterns. More recent developments and variants of this perspective are reviewed by \citet{krotov2025modern}. These works primarily formulate memories as vectors in a Euclidean ambient space, whereas our goal is to store and retrieve graph data whose relevant information is encoded through relational and spectral structure.

\paragraph{Modern Hopfield networks and log-sum-exponential energies.}
A particularly relevant modern formulation is the continuous-state Hopfield network of \citet{ramsauer2020hopfield}, which uses a log-sum-exponential (LSE) energy and connects associative memory retrieval to the attention mechanism in Transformers. Its retrieval rule takes the form of a softmax-weighted aggregation of stored vector patterns, with increasing inverse temperature concentrating the weights around the most similar memories. This formulation provides the main energy-based motivation for our construction. Recent theoretical work has further investigated the exponential storage capacity of dense associative memories; for example, \citet{lucibello2024exponential} establishes exponential-capacity phenomena for DAMs with suitable nonlinear energies. Our Graph-DAM retains the Gibbs aggregation principle but replaces Euclidean vector similarity with a discrepancy induced by graph Dirichlet energy. In particular, rather than comparing vector memories through Euclidean inner products, we compare graph Laplacians through the operator norm, which has a variational interpretation as the largest discrepancy between their Dirichlet energies over unit graph signals.
\paragraph{Alternative energy functions for dense associative memories.}
Beyond the LSE formulation, recent work has explored how different energy or separation functions affect retrieval behavior and the geometry of associative memories. \citet{hoover2026dense}, for example, studies dense associative memory with Epanechnikov energy. More broadly, these developments emphasize that the energy function determines the similarity notion and retrieval dynamics of the memory system. Our work follows this general perspective but addresses a different aspect of the problem: the geometry of the \emph{objects being stored}. Rather than modifying the energy while continuing to operate on Euclidean vector patterns, we construct an associative memory energy directly for graphs. The resulting similarity is defined through graph Laplacians and Dirichlet energy, so that retrieval is sensitive to graph signal and spectral geometry rather than accumulated coordinate-wise or edge-wise discrepancies in Euclidean DAMs.
\paragraph{Associative memories beyond Euclidean vector spaces.}
Most classical and modern associative memory models assume that stored patterns belong to a Euclidean vector space. Very recent work has begun extending associative memories to non-Euclidean objects. \citet{tankala2026dense} considers dense associative memories whose memories are Gaussian distributions, while \citet{shi2026intrinsic} develops an intrinsic dense associative memory for patterns lying on Riemannian manifolds. These works demonstrate that associative memories can be adapted to spaces in which Euclidean vector similarity is no longer the appropriate geometry. Graphs introduce a distinct challenge: their information is inherently relational and is represented through pairwise connectivity and the action of graph operators on signals. Accordingly, our construction does not embed graphs into a generic Euclidean or manifold representation. Instead, it exploits the graph Dirichlet energy and its Laplacian operator representation to define the memory geometry directly. The resulting retrieval operator is a Gibbs-weighted barycenter of stored graph Laplacians and therefore remains a valid graph Laplacian, allowing the memory dynamics to operate directly within the space of graph representations.

\paragraph{Relation to our Graph-DAM.}
Graph-DAM can therefore be viewed as extending the modern energy-based associative memory paradigm along the \emph{data geometry} axis. Classical Hopfield networks and DAMs store vector-valued patterns; modern Hopfield networks introduce LSE energies and softmax retrieval; recent DAM variants investigate alternative energies; and recent non-Euclidean approaches generalize associative memories to distributions or manifolds. Our work instead develops an associative memory specifically for graph-structured objects. By defining the LSE energy using operator-norm discrepancies between graph Laplacians, Graph-DAM associates memories according to their worst-case Dirichlet energy discrepancy. This distinction is important because vectorizing a Laplacian and applying a Euclidean DAM reduces similarity to a Euclidean/Frobenius comparison, which aggregates entrywise discrepancies and need not preserve the graph's dominant spectral structure. In contrast, the proposed formulation directly targets the spectral geometry underlying graph smoothness, diffusion, and community structure, while retaining the characteristic storage and retrieval properties of modern dense associative memories.

\section{Proofs}
In this section, we include the proofs of all theoretical results in Section \ref{sec:thms}. 

\subsection{Proof of Theorem \ref{thm:exp-storage}}
The proof contains three parts as follows:
\begin{itemize}
    \item We first prove for a properly chosen radius $r=r_0$, the basins $B_i(r_0)$ and $B_j(r_0)$ for any $i\neq j$ are disjoint.
    \item We then show the retrieval operator $\mcal T$ is a self map on the basins.
    \item Last, we establish the contraction property of this map $\mcal T$.
\end{itemize}
Hence using Banach’s fixed-point theorem, we show $\mcal T$ has a unique fixed point in the basin such that every initialization in the basin converges to it. The number of such memories can be exponentially large in $n$.
\begin{proof}
For \(i\neq j\), define
\[
S_{ij}
=
\sum_{u<v}
\left(W_{i,uv}-W_{j,uv}\right)^2.
\]
Since \(W_{i,uv},W_{j,uv}\) are independent \({\rm Unif}[0,1]\),
\[
\mathbb E\left[
\left(W_{i,uv}-W_{j,uv}\right)^2
\right]
=
\frac16.
\]
Note that the summands are independent and bounded in \([0,1]\). Thus, with
\(
m=\binom n2,
\)
Hoeffding's inequality (see, for example, \citet[Proposition 2.5]{wainwright2019high}) gives
\[
\mathbb P\left(S_{ij}\le \frac m{12}\right)
\le
\exp\left(-\frac m{72}\right).
\]
By a union bound,
\[
\mathbb P\left(
\exists i\neq j:S_{ij}\le\frac m{12}
\right)
\le
N^2e^{-m/72}.
\]
Since \(N\le \sqrt{\delta}e^{c_{\rm sep}n}\),
\[
N^2e^{-m/72}
\le
\delta
\exp\left\{
2c_{\rm sep}n-\frac{n(n-1)}{144}
\right\}.
\]
Hence fix any \(c_{\rm sep}<1/288\). Then for all \(n\ge2\),
\[
2c_{\rm sep}n-\frac{n(n-1)}{144}\le0.
\]
This implies
\[
\mathbb P\left(
S_{ij}\ge\frac m{12}
\text{ for all }i\neq j
\right)
\ge
1-\delta.
\]

Now, we focus on the above event with the probability at least $1-\delta$. Let
\(
\Delta^{(ij)}=L_i-L_j.
\)
Its off-diagonal entries satisfy
\[
\Delta^{(ij)}_{uv}
=
-\left(W_{i,uv}-W_{j,uv}\right),
\qquad u\neq v.
\]
Then, we obtain:
\[
\|\Delta^{(ij)}\|_{\F}^2
\ge
2S_{ij}
\ge
\frac m6.
\]
Using \(\|A\|_{\op}\ge \|A\|_{\F}/\sqrt n\), it yields that
\[
\|L_i-L_j\|_{\op}^2
\ge
\frac{m}{6n}.
\]
Thus
\[
d_n(G_i,G_j)^2
=
\frac1n\|L_i-L_j\|_{\op}^2
\ge
\frac{m}{6n^2}
=
\frac{n-1}{12n}
\ge
\frac1{24}=:\delta_0^2.
\]
Then, we can set
\(
r_0=\frac{1}{4}\delta_0
\)
such that
\[
d_n(G_i,G_j)\ge \delta_0>2r_0,
\]
so the basins are disjoint:
\[
B_i(r_0)\cap B_j(r_0)=\emptyset.
\]

Fix \(i\) and take \(G\in B_i(r_0)\). For \(j\neq i\),
\[
d_n(G,G_j)
\ge
d_n(G_i,G_j)-d_n(G,G_i)
\ge
\delta_0-r_0.
\]
Therefore
\[
d_n(G,G_j)^2-d_n(G,G_i)^2
\ge
(\delta_0-r_0)^2-r_0^2
=
\frac{\delta_0^2}{2}
=:
\gamma_0.
\]
Then for the Gibbs weights, we have:
\[
\frac{w_j(G)}{w_i(G)}
\le
e^{-\beta n\gamma_0},
\]
and hence
\[
\sum_{j\neq i}w_j(G)
\le
(N-1)e^{-\beta n\gamma_0}.
\]

Since all weights lie in \([0,1]\), each degree is at most \(n-1\), so
\(
\|L_i\|_{\op}\le 2(n-1).
\)
Therefore
\[
d_n(G_j,G_i)
=
n^{-1/2}\|L_j-L_i\|_{\op}
\le
4\sqrt n.
\]
Noting that 
\[
\mcal T(G)-L_i
=
\sum_{j\neq i}w_j(G)(L_j-L_i),
\]
we obtain
\begin{align}\label{oneupdatebound}
d_n(\mcal T(G),G_i)
\le
4\sqrt n (N-1)e^{-\beta n\gamma_0}.
\end{align}
By the definition of \(\beta_0\) and $N$, for $\beta\ge \beta_0$, the distance is bounded by $4\sqrt n e^{-(\beta \gamma_0-c_{\text{sep}})n}\rightarrow 0$ as $n\rightarrow\infty$. Hence there exsits $n_0\in\mbb Z_{+}$ such that for all $n\ge n_0$, $d_n(\mcal T(G),G_i)\le r_0$. Then, this implies that $\mcal T$ is a self map on the basin:
\[
\mcal T(B_i(r_0))\subseteq B_i(r_0).
\]

It remains to prove contraction of the map $\mcal T$. Let
\[
q_j(G)=\frac{w_j(G)}{w_i(G)}
=
e^{-\beta n g_j(G)},
\]
where
\(
g_j(G)
=
d_n(G,G_j)^2-d_n(G,G_i)^2.
\)
For \(G,G'\in B_i(r_0)\),
\[
|g_j(G)-g_j(G')|
\le
C\sqrt n\, d_n(G,G')
\]
for a universal constant \(C>0\). Also \(g_j(G)\ge\gamma_0\) on \(B_i(r_0)\). Hence
\[
|q_j(G)-q_j(G')|
\le
C\beta n^{3/2}e^{-\beta n\gamma_0}d_n(G,G').
\]
Using
\[
\mcal T(G)-L_i
=
\frac{
\sum_{j\neq i}q_j(G)(L_j-L_i)
}{
1+\sum_{j\neq i}q_j(G)
},
\]
and \(d_n(G_j,G_i)\le4\sqrt n\), we obtain
\begin{align}\label{contractionbound}
d_n(\mcal T(G),\mcal T(G'))
\le
C\beta n^2 (N-1)e^{-\beta n\gamma_0}
d_n(G,G').
\end{align}
By the definition of \(\beta_0\) and $N$, for $\beta\ge \beta_0$, there exists $n_0\in\mbb Z_{+}$ such that for all $n\ge n_0$, similarly as above, this coefficient $C\beta n^2 (N-1)e^{-\beta n\gamma_0}$ is less than one. Therefore \(\mcal T\) is a contraction on
\(B_i(r_0)\).

By Banach's fixed-point theorem, \(\mcal T\) has a unique fixed point in \(B_i(r_0)\), and every
initialization in \(B_i(r_0)\) converges to it. Since this holds for every \(i\) on an event of
probability at least \(1-\delta\), all memories are stored with probability at least \(1-\delta\).
\end{proof}

\subsection{Proof of Theorem \ref{thm:retrievalproperty}}
\begin{proof}
By Theorem~\ref{thm:exp-storage}, with probability at least \(1-\delta\), the sampled graph
memories satisfy the uniform separation bound
\[
    \min_{i\neq j} d_n(G_i,G_j)\ge \delta_0,
    \qquad
    \delta_0:=\frac1{\sqrt{24}}.
\]
Throughout the proof, we work on this event in Theorem \ref{thm:exp-storage}. Recall 
\[
    r_0:=\frac{\delta_0}{4},
    \qquad
    \gamma_0:=\frac{\delta_0^2}{2}.
\]
According to \eqref{contractionbound} in Theorem \ref{thm:exp-storage}, we obtain: for $G,G'\in B_i(r_0)$,
\[
    d_n(\mcal T(G),\mcal T(G'))
    \le
    C\beta n^2 (N-1)e^{-\beta n\gamma_0}
    d_n(G,G').
\]
Define
\(
    \rho_n
    :=
    C\beta n^2 (N-1)e^{-\beta n\gamma_0}.
\)
Since \(N\le \sqrt{\delta}e^{c_{\rm sep}n}\),
\[
    \rho_n
    \le
    C\beta n^2\sqrt{\delta}
    \exp\{-(\beta\gamma_0-c_{\rm sep})n\}.
\]
Because \(\beta\gamma_0-c_{\rm sep}>0\), we have \(\rho_n\to0\). Therefore, for all
sufficiently large \(n\),
\(
    \rho_n<1.
\)
Thus \(\mcal T\) is a contraction on \(B_i(r_0)\) with some contraction constant $\rho\in(0,1)$. Moreover, for the iterates \(L_{G^{(t+1)}}=\mcal T(G^{(t)})\),
\[
    d_n(G^{(t)},G_i^\star)
    \le
    \rho^t d_n(G^{(0)},G_i^\star).
\]
Since \(G^{(0)},G_i^\star\in B_i(r_0)\),
\[
    d_n(G^{(0)},G_i^\star)
    \le
    d_n(G^{(0)},G_i)+d_n(G_i,G_i^\star)
    \le
    2r_0.
\]
Therefore, we have:
\[
    d_n(G^{(t)},G_i^\star)
    \le
    2r_0\rho^t.
\]
This proves convergence:
\(
    G^{(t)}\to G_i^\star
\)
in the normalized operator metric. Finally, to ensure
\[
    d_n(G^{(t)},G_i^\star)\le \varepsilon,
\]
it is enough that
\[
    2r_0\rho^t\le \varepsilon.
\]
Taking logarithms gives
\[
    t
    \ge
    \frac{\log(2r_0/\varepsilon)}{\log(1/\rho)}.
\]
This proves all claims.
\end{proof}

\subsection{Proof of Theorem \ref{thm:retrievalerror}}
\begin{proof}
From Theorem \ref{thm:exp-storage}, with high-probability at least $1-\delta$, the separation holds that
\[
\min_{i\neq j}d_n(G_i,G_j)\ge \delta_0:=\frac{1}{\sqrt{24}}.
\]
According to \eqref{oneupdatebound} in the proof of Theorem \ref{thm:exp-storage}, we already obtained:
\[
d_n(T(G),G_i)
\le
4\sqrt n (N-1)e^{-\beta n\gamma_0}.
\]
Using
\(
N\le \sqrt{\delta}e^{c_{\rm sep}n},
\)
it implies
\[
d_n(T(G),G_i)
\le
4\sqrt{\delta n}
\exp\{-(\beta\gamma_0-c_{\rm sep})n\}
=
4\sqrt{\delta n}e^{-\eta n},
\]
where we define $\eta:=\beta\gamma_0-c_{\text{sep}}>0$ with constants $\gamma_0$ and $c_{\text{sep}}$ defined in the proof of Theorme \ref{thm:exp-storage}. This proves the claim.
\end{proof}

\subsection{Proof of Theorem \ref{thm:spectralerror}}

\begin{proof}
Throughout the proof, we work on the high-probability event in
Theorems \ref{thm:exp-storage} and \ref{thm:retrievalproperty} on which the separation condition holds, the basin
$B_i(r_0)$ is invariant under the retrieval map $\mathcal{T}$, and
$\mathcal{T}$ is contractive on $B_i(r_0)$.

Recall that the normalized operator metric is defined by
\[
d_n(G,G')
=
n^{-1/2}\|L_G-L_{G'}\|_{\op}.
\]
Using \eqref{contractionbound}, for any
$G,G'\in B_i(r_0)$,
\[
d_n\bigl(\mathcal{T}(G),\mathcal{T}(G')\bigr)
\le
C\beta n^2 (N-1)e^{-\beta n\gamma_0}
d_n(G,G').
\]
Define
\(
\widetilde{\rho}_n
:=
C\beta n^2 (N-1)e^{-\beta n\gamma_0}.
\)
Since
\(
N
=
\left\lfloor
\sqrt{\delta}\,e^{c_{\mathrm{sep}}n}
\right\rfloor,
\)
we then have
\[
\widetilde{\rho}_n
\asymp
\beta n^2\sqrt{\delta}\,
\exp\{-(\beta\gamma_0-c_{\mathrm{sep}})n\}.
\]
By the choice of $\beta$ in Theorem~\ref{thm:exp-storage},
$\beta\gamma_0-c_{\mathrm{sep}}>0$, and hence
\[
\widetilde{\rho}_n=o(1)
\qquad\text{as }n\to\infty.
\]
In particular, $\widetilde{\rho}_n<1$ for all sufficiently large $n$. For notational convenience, let
\[
L^{(t)}:=L_{G^{(t)}}.
\]
Since $L_i^\star$ is the unique fixed point of $\mathcal{T}$ in
$B_i(r_0)$, we have
\[
\mathcal{T}(G_i^\star)=L_i^\star.
\]
Iterating the contraction inequality yields
\[
d_n(G^{(t)},G_i^\star)
\le
\widetilde{\rho}_n^t
d_n(G^{(0)},G_i^\star).
\]
Since both $G^{(0)}$ and $G_i^\star$ belong to $B_i(r_0)$,
the triangle inequality gives
\[
d_n(G^{(0)},G_i^\star)
\le
d_n(G^{(0)},G_i)+d_n(G_i,G_i^\star)
\le
2r_0.
\]
We then obtain:
\[
d_n(G^{(t)},G_i^\star)
\le
2r_0\widetilde{\rho}_n^t.
\]
By the definition of $d_n$, this is equivalent to
\begin{equation}
\label{eq:operator-recovery}
\|L^{(t)}-L_i^\star\|_{\op}
\le
2r_0\sqrt{n}\,\widetilde{\rho}_n^t.
\end{equation}
In the following, we establish the three claims separately.

\paragraph{1. Eigenvalue error.}
Since both $L^{(t)}$ and $L_i^\star$ are real symmetric matrices,
Weyl's eigenvalue perturbation inequality implies that, for every
$k\in[n]$,
\[
\left|
\lambda_k(L^{(t)})-\lambda_k(L_i^\star)
\right|
\le
\|L^{(t)}-L_i^\star\|_{\op}.
\]
Applying \eqref{eq:operator-recovery}, we obtain
\[
\left|
\lambda_k(L^{(t)})-\lambda_k(L_i^\star)
\right|
\le
2r_0\sqrt{n}\,\widetilde{\rho}_n^t,
\qquad k\in[n].
\]
This proves the first claim.

\paragraph{2. Eigenvector error.}
Let
\[
U_K^\star
=
[u_1^\star,\ldots,u_K^\star]
\]
denote the matrix whose columns are the eigenvectors corresponding to
the first $K$ eigenvalues of $L_i^\star$, and let
\[
U_K^{(t)}
=
[u_1^{(t)},\ldots,u_K^{(t)}]
\]
denote the analogous eigenvector matrix for $L^{(t)}$.
By the eigengap assumption,
\[
\mu_K
:=
\lambda_{K+1}(L_i^\star)-\lambda_K(L_i^\star)
>0.
\]

Applying the Davis-Kahan $\sin\Theta$ theorem in its
population-eigengap form (see, for example, \citet[Theorem 10]{cai2013sparse}) gives
\[
\left\|
\sin\Theta\bigl(U_K^{(t)},U_K^\star\bigr)
\right\|_{\op}
\le
\frac{
2\|L^{(t)}-L_i^\star\|_{\op}
}{
\mu_K
}.
\]
Combining this inequality with
\eqref{eq:operator-recovery} yields
\[
\left\|
\sin\Theta\bigl(U_K^{(t)},U_K^\star\bigr)
\right\|_{\op}
\le
\frac{
4r_0\sqrt{n}\,\widetilde{\rho}_n^t
}{
\mu_K
}.
\]
Thus the invariant subspace generated by the first $K$ Laplacian
eigenvectors is recovered at the same rate
$\widetilde{\rho}_n^t$.

\paragraph{3. Graph diffusion error.}
Fix any diffusion time $\tau>0$. By the Duhamel formula,
\[
e^{-\tau L^{(t)}}-e^{-\tau L_i^\star}
=
-\int_0^\tau
e^{-(\tau-s)L^{(t)}}
\bigl(L^{(t)}-L_i^\star\bigr)
e^{-sL_i^\star}
\,ds.
\]
Taking operator norms, we have:
\begin{align*}
\left\|
e^{-\tau L^{(t)}}-e^{-\tau L_i^\star}
\right\|_{\op}
&\le
\int_0^\tau
\left\|
e^{-(\tau-s)L^{(t)}}
\right\|_{\op}
\left\|
L^{(t)}-L_i^\star
\right\|_{\op}
\left\|
e^{-sL_i^\star}
\right\|_{\op}
\,ds.
\end{align*}
Both $L^{(t)}$ and $L_i^\star$ are graph Laplacians and are therefore
positive semidefinite. Hence, for every positive semidefinite matrix
$L$ and every $s\ge0$,
\[
\|e^{-sL}\|_{\op}
=
\max_{1\le k\le n} e^{-s\lambda_k(L)}
\le 1.
\]
It yields that
\begin{align*}
\left\|
e^{-\tau L^{(t)}}-e^{-\tau L_i^\star}
\right\|_{\op}
&\le
\int_0^\tau
\left\|
L^{(t)}-L_i^\star
\right\|_{\op}
\,ds \\
&=
\tau
\left\|
L^{(t)}-L_i^\star
\right\|_{\op}.
\end{align*}
Using \eqref{eq:operator-recovery}, we conclude that
\[
\left\|
e^{-\tau L^{(t)}}-e^{-\tau L_i^\star}
\right\|_{\op}
\le
2\tau r_0\sqrt{n}\,
\widetilde{\rho}_n^t.
\]
This proves the graph diffusion error bound and completes the proof.
\end{proof}

\section{Additional Experiments}\label{sec:app_experiment}

\subsection{Real-World Human Activity Sensor Data}
\label{sec:har}

We further evaluate our Graph-DAM retrieval on the UCI Human Activity Recognition (HAR) dataset, which provides continuous inertial measurements collected from smartphones worn by human subjects while performing daily activities. Each recording contains six synchronized sensor channels: three-axis linear acceleration and three-axis angular velocity measured by the smartphone accelerometer and gyroscope, respectively. The dataset contains multiple subjects performing activities such as walking, walking upstairs, walking downstairs, sitting, standing, and lying. The observations here are not originally given as graphs. Nevertheless, the multivariate sensor recordings are naturally relational and non-Euclidean: a human activity is characterized not only by individual acceleration and angular-velocity signals, but also by the structured dependencies among sensor channels and how these dependencies evolve across different stages of the motion. A graph therefore provides a natural representation of this relational structure, with nodes representing temporally localized sensor components and edges encoding their dependencies. This motivates treating each recording as a graph-valued memory rather than as a vector in an ambient Euclidean space, and provides a real-data setting for evaluating whether the operator geometry underlying Graph-DAM can better preserve and retrieve such structured relational information.

\paragraph{Temporal-sensor graph construction.} We construct $N=30$ stored memories from distinct real activity intervals. For each interval, let
\[
X_i=(x_{itc})\in\mathbb{R}^{T_i\times 6}
\]
denote its six-channel sensor recording, where $t$ indexes time and $c$ indexes the accelerometer or gyroscope channel. Human motion is intrinsically multivariate: an activity is characterized not only by the magnitude of individual sensor channels, but also by how different acceleration and angular-velocity components co-vary and how these relationships evolve over the course of the motion. We therefore represent each recording through a temporal-sensor graph.

Specifically, we divide each activity interval into $10$ consecutive temporal blocks. Within every block, each of the six sensor channels defines one node, resulting in a common set of
\(
n=10\times 6=60
\)
aligned nodes. Thus, a node represents a particular sensor component at a particular stage of the activity, such as the $x$-axis acceleration during the first temporal block or the $z$-axis angular velocity during a later block. The local waveform associated with each node is resampled to a common resolution and standardized. We compute the absolute correlation between every pair of local waveforms and retain a sparse $10$-nearest-neighbor connectivity pattern. This yields a weighted adjacency matrix
\(
A_i\in\mathbb{R}^{60\times60},
\)
whose edge weights encode dependencies between sensor components across both channels and temporal portions of the activity. We then form the combinatorial graph Laplacian
\[
L_i=D_i-A_i.
\]

\paragraph{Why graph memory is needed.}
This construction is motivated by the fact that the identity of a human-motion recording is not determined solely by isolated sensor amplitudes. Coordinated physical motion induces structured relationships among acceleration and rotation signals. For example, two recordings may have comparable marginal acceleration or frequency statistics while differing substantially in which sensor components move together and at which stages of the activity. Conversely, moderate perturbations to individual measurements need not destroy the underlying coordination pattern of the motion. The relevant memory object is therefore naturally relational: it describes a structured pattern of interactions among temporally localized sensor signals.

This setting also illustrates a limitation of treating graph memories as ordinary Euclidean vectors. Vectorizing $L_i$ and comparing memories through Euclidean inner products treats the entries of the Laplacian as coordinates in an ambient vector space. Such a representation does not explicitly respect the operator structure of the Laplacian. In contrast, Graph-DAM compares memories through the induced graph operators. In particular, differences between Laplacians are measured through their operator norm, which captures the largest discrepancy in their action on graph signals. Hence, retrieval is driven by differences in the collective interaction structure encoded by the graph rather than by an entrywise comparison alone. The HAR experiment therefore provides a real-data setting in which the memory itself is constructed from structured multivariate dependencies and motivates a non-Euclidean graph-memory representation.

\paragraph{Query construction and baselines.}
For each stored graph memory, we generate a corrupted query by uniformly removing $20\%$ of its existing edges. From the resulting adjacency matrix $A_i^{(q)}$, we recompute
\[
L_i^{(q)}=D_i^{(q)}-A_i^{(q)}.
\]
This corruption removes part of the observed dependency structure while leaving the target memory unchanged, so the retrieval problem is to reconstruct the correct underlying temporal-sensor organization from incomplete relational information.

We compare Graph-DAM against two Euclidean baselines. First, Eu-DAM applies the classical Euclidean associative-memory update directly to $\operatorname{vec}(L_i)$ and therefore receives exactly the same graph information as Graph-DAM, but treats the Laplacian as an ordinary Euclidean vector. This comparison isolates the effect of the geometry used for retrieval.

Second, Raw Eu-DAM operates directly on a conventional non-graph representation of the original sensor recording. For each of the six inertial channels, we extract standard time- and frequency-domain summaries, including the mean, standard deviation, root-mean-square amplitude, median, interquartile range, median absolute deviation, temporal-difference statistics, zero-crossing rate, spectral centroid, spectral entropy, dominant frequency, and low-, middle-, and high-frequency spectral energies. We additionally include the corresponding summaries of the accelerometer and gyroscope magnitudes. For the raw query, $20\%$ of the temporal observations are removed, the missing observations are reconstructed by interpolation, and the same feature representation is recomputed. Thus, Raw Eu-DAM tests whether conventional vector-valued summaries of the sensor recording are sufficient for associative retrieval without explicitly representing the dependency structure among temporally localized sensor signals.

\begin{figure}[!t]
    \centering
    \includegraphics[width=0.62\linewidth]{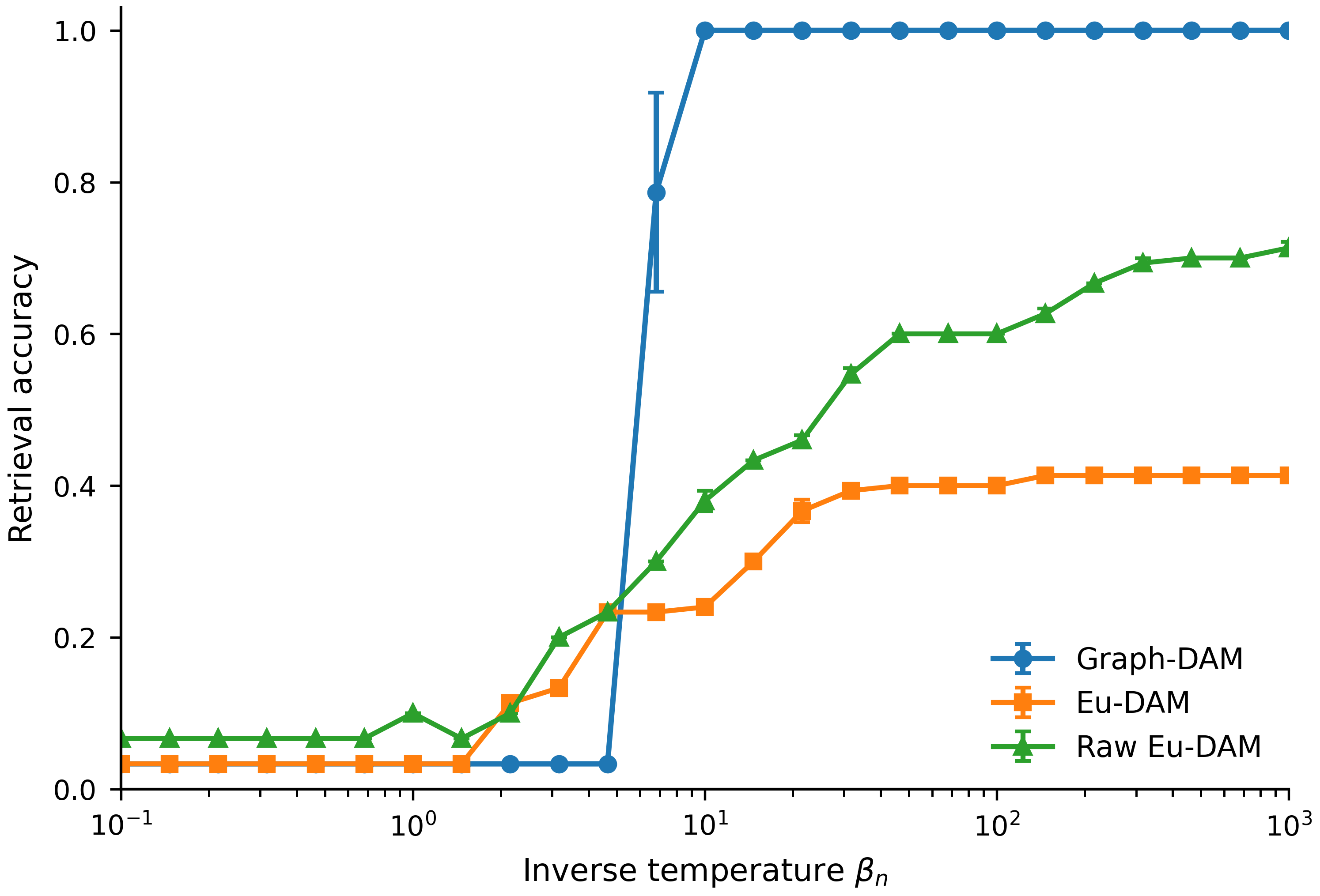}
    \caption{UCI HAR data retrieval performance. Retrieval accuracy of Graph-DAM, Eu-DAM, and Raw Eu-DAM versus inverse temperature $\beta_n$.}
    \label{fig:har_acc}
\end{figure}

\begin{figure}[!t]
    \centering
    \includegraphics[width=\linewidth]{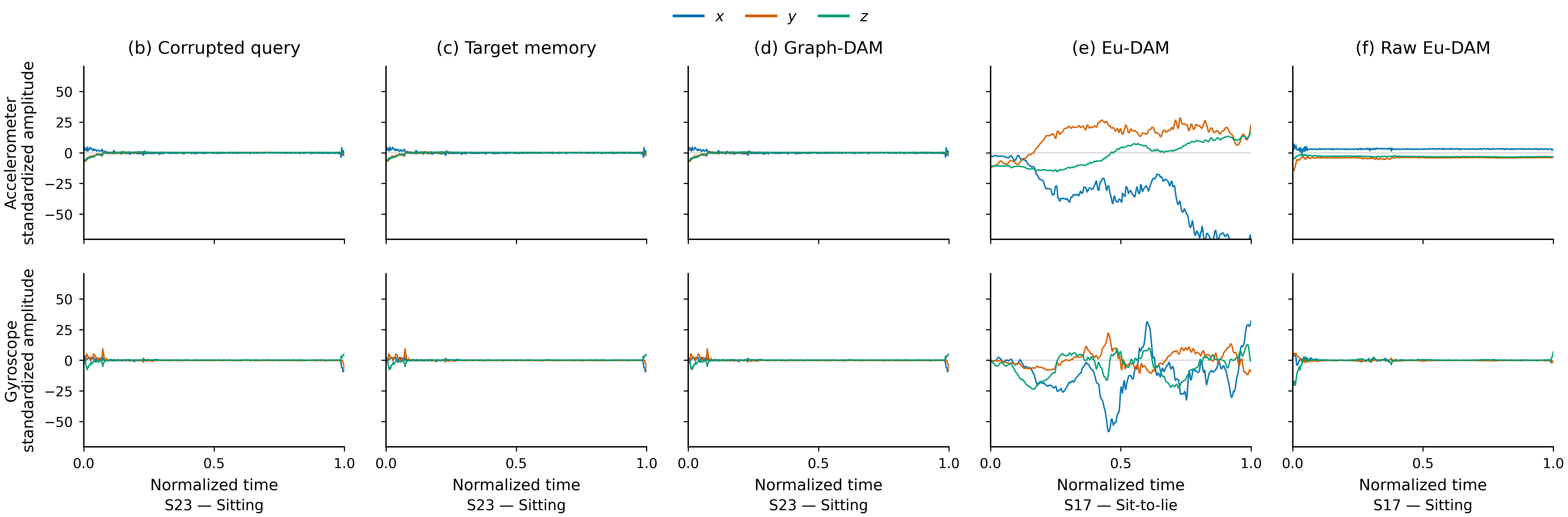}
    \caption{\textbf{UCI HAR retrieval example visualized in the original sensor space.}
    The top and bottom rows show the three-axis accelerometer and gyroscope signals, respectively. From left to right, we show the corrupted query, its true target memory, and the memories retrieved by Graph-DAM, Eu-DAM, and Raw Eu-DAM. The query corresponds to a sitting recording from subject S23. Graph-DAM correctly retrieves the target S23 sitting memory, whereas Eu-DAM retrieves an S17 sit-to-lie recording and Raw Eu-DAM retrieves an S17 sitting recording. All signals are plotted over normalized time and standardized using the target-memory statistics to allow direct comparison of their temporal patterns and amplitudes.}
    \label{fig:har_signals}
\end{figure}

\begin{figure}[!t]
    \centering
    \includegraphics[width=\linewidth]{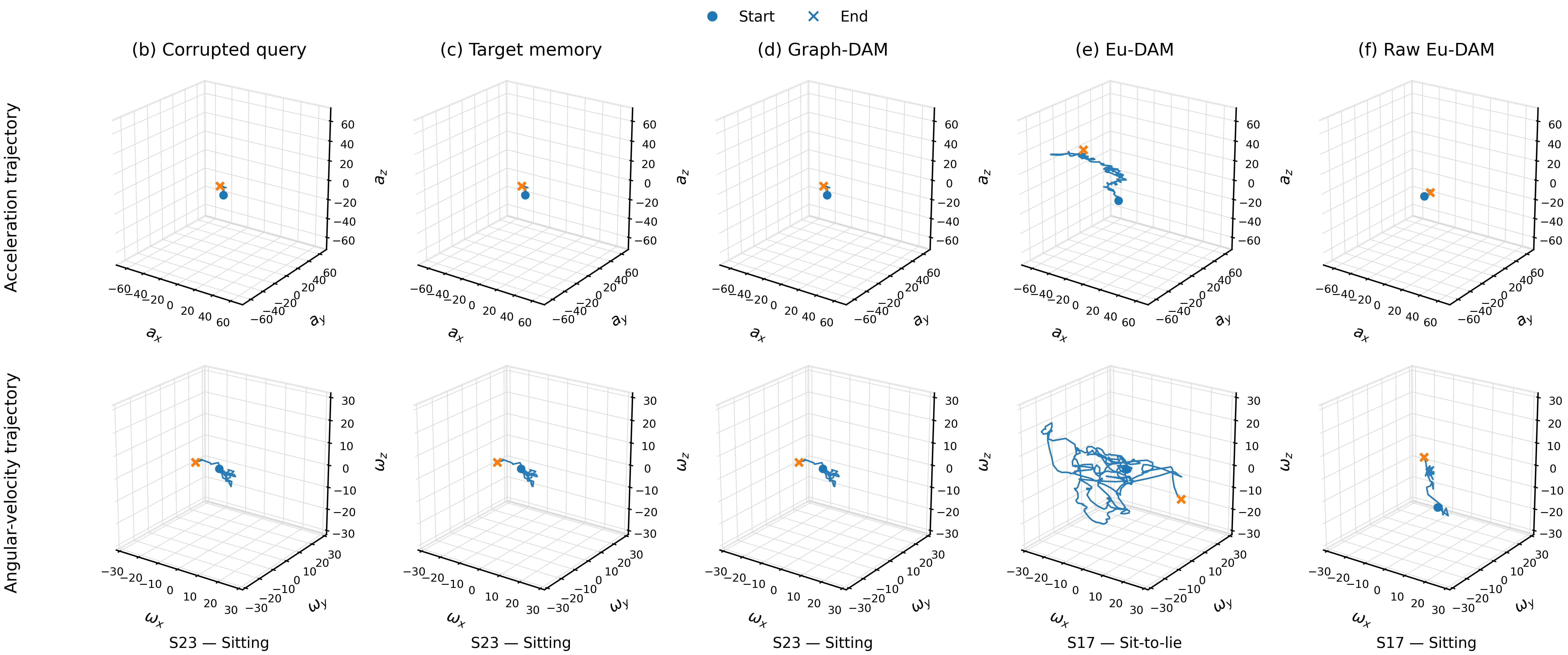}
    \caption{\textbf{Three-dimensional sensor trajectories for the UCI HAR retrieval example.}
    The top and bottom rows visualize the three-dimensional accelerometer trajectories $(a_x,a_y,a_z)$ and angular-velocity trajectories $(\omega_x,\omega_y,\omega_z)$, respectively, for the corrupted query, target memory, and memories retrieved by Graph-DAM, Eu-DAM, and Raw Eu-DAM. Graph-DAM recovers the true S23 sitting memory, while Eu-DAM retrieves an S17 sit-to-lie memory and Raw Eu-DAM retrieves an S17 sitting memory. The trajectories provide a direct visualization in the original sensor domain of the motion patterns associated with the retrieved graph memories. Circles and crosses indicate the starting and ending points of each trajectory, respectively.}
    \label{fig:har_3d}
\end{figure}

\textbf{Main results.} Figure \ref{fig:har_acc} shows that Graph-DAM achieves substantially higher retrieval accuracy than both Eu-DAM and Raw Eu-DAM as the inverse temperature $\beta_n$ increases. In particular, Graph-DAM rapidly reaches perfect retrieval accuracy and remains at $100\%$ for larger values of $\beta_n$, whereas Eu-DAM plateaus at approximately $41\%$ and Raw Eu-DAM reaches approximately $71\%$. This difference is particularly pronounced in the high-inverse-temperature regime, where the associative memory increasingly concentrates on individual stored memories.

The advantage of Graph-DAM is consistent with the relational structure of multivariate inertial sensor data. Although the original UCI HAR observations are multivariate time series rather than graphs, a human activity is characterized not only by the values of individual accelerometer and gyroscope channels, but also by how different sensor components co-vary across different stages of the motion. Our graph construction explicitly represents these dependencies through temporally localized sensor nodes and weighted edges measuring their pairwise relationships. Consequently, each memory encodes a structured pattern of interactions among sensor channels over time. Graph-DAM performs retrieval using the operator geometry of the corresponding graph Laplacians and therefore compares memories through their induced relational structure. In contrast, Eu-DAM vectorizes the Laplacian and treats its entries in an ambient Euclidean space, while Raw Eu-DAM operates directly on non-graph sensor features. The results in Figure \ref{fig:har_acc} suggest that preserving the non-Euclidean relational structure of the sensor recordings can be important for associative retrieval.

Figures \ref{fig:har_signals} and \ref{fig:har_3d} provide a concrete retrieval example. The target memory corresponds to a sitting recording from subject S23. Graph-DAM correctly retrieves this target memory, whereas Eu-DAM retrieves a sit-to-lie recording from subject S17 and Raw Eu-DAM retrieves a different sitting recording from subject S17. Figure \ref{fig:har_signals} visualizes the corresponding memories in the original sensor space through their three-axis accelerometer and gyroscope signals. The Graph-DAM retrieval therefore reproduces the temporal sensor pattern of the true target, while the Eu-DAM retrieval corresponds to a different motion transition. Interestingly, Raw Eu-DAM identifies the correct activity category but retrieves the wrong stored memory, indicating that similarity at the level of conventional sensor features does not necessarily distinguish the finer subject- and recording-specific dependency structure required for exact memory retrieval.

The same behavior can be viewed geometrically in Figure \ref{fig:har_3d}, which plots the three-dimensional accelerometer and angular-velocity trajectories associated with the query and retrieved memories. The trajectory of the memory selected by Graph-DAM coincides with that of the true S23 target, whereas the memories selected by the two Euclidean baselines exhibit different motion trajectories. Together with the retrieval-accuracy results, these examples illustrate that the graph representation retains structured cross-sensor and temporal relationships that are not fully captured by entrywise Euclidean comparison or by the raw feature representation. This provides a complementary real-data example to the airline and protein experiments: even when the observations are not originally given as graphs, Graph-DAM can be useful when the data contain an underlying non-Euclidean relational structure that can be represented through interactions among their components.

\subsection{Real-World Wearable Sensor Data}
\label{sec:dsa}

We further evaluate Graph-DAM on the UCI Daily and Sports Activities (DSA) dataset, which contains multivariate inertial measurements collected from wearable sensors placed at multiple locations on the human body while subjects perform a diverse collection of daily and sports activities. Sensors are placed on the torso, right arm, left arm, right leg, and left leg, with each location providing three-axis accelerometer, gyroscope, and magnetometer measurements. The dataset contains multiple subjects performing activities ranging from relatively stationary motions, such as sitting and standing, to highly dynamic activities, such as running, rowing, jumping, and playing basketball.

Although the original observations are multivariate time series rather than graphs, the distributed sensing configuration gives the data an explicit relational interpretation. A physical activity is characterized not only by the motion measured at each individual body location, but also by how the torso and limbs move together. For example, jumping induces coordinated acceleration patterns across the torso and both legs, whereas standing produces a very different pattern of cross-body dependencies. A graph therefore provides a natural representation of these spatially distributed interactions, with nodes representing sensor components at different body locations and edges encoding their statistical dependencies. This provides a second real-world sensor setting for evaluating whether the operator geometry underlying Graph-DAM can preserve and retrieve structured relational information that is not naturally represented by an ambient Euclidean vector.

\paragraph{Body-sensor graph construction.}
We construct $N=30$ stored memories from real subject--activity recordings. For each memory, let
\(
X_i=(x_{itc})\in\mathbb{R}^{T_i\times 45}
\)
denote the corresponding sensor recording, where $t$ indexes time and $c$ indexes the sensor channel. The $45$ channels arise from five body locations---torso, right arm, left arm, right leg, and left leg---with nine measurements at each location corresponding to three-axis acceleration, angular velocity, and magnetic-field measurements.

We associate one graph node with each sensor channel, resulting in a common set of
\(
n=5\times 9=45
\)
aligned nodes. Thus, every node has a direct physical interpretation, such as the $x$-axis acceleration of the torso or the $z$-axis angular velocity of the right leg. After standardizing the temporal waveform associated with each node, we compute the absolute correlation between every pair of sensor channels and retain a sparse $8$-nearest-neighbor connectivity pattern. This produces a weighted adjacency matrix
\(
A_i\in\mathbb{R}^{45\times45},
\)
whose edge weights encode dependencies among sensor measurements both within and across different body locations. We then form the combinatorial graph Laplacian
\[
L_i=D_i-A_i.
\]

\paragraph{Why graph memory is needed.}
This construction is motivated by the fact that human activities are inherently coordinated whole-body motions. The identity of a recording is not determined solely by the marginal acceleration, angular velocity, or magnetic-field signal measured at an individual sensor. Rather, it also depends on structured relationships among spatially distributed body components. Two recordings may exhibit similar local motion statistics while differing in how the torso, arms, and legs move together. Conversely, individual sensor measurements may vary across repetitions of an activity while the underlying coordination pattern remains similar. The relevant memory object is therefore naturally relational: it represents a structured pattern of dependencies among body-mounted sensors.

The DSA dataset makes this relational interpretation particularly concrete because the graph nodes correspond to physically distinct sensor components distributed over the body. Vectorizing $L_i$ and treating its entries as ordinary Euclidean coordinates does not explicitly respect the operator structure of the resulting graph. Graph-DAM instead compares memories through the induced graph operators. In particular, the operator norm measures the largest discrepancy in the action of two Laplacians on graph signals, so retrieval is driven by differences in the collective dependency structure among body-mounted sensors rather than by entrywise similarity alone. The experiment therefore tests whether this non-Euclidean geometry improves associative retrieval when the memory encodes coordinated physical motion.

\paragraph{Query construction and baselines.}
For each stored graph memory, we generate a corrupted query by uniformly removing $20\%$ of its existing edges. From the resulting adjacency matrix $A_i^{(q)}$, we recompute
\[
L_i^{(q)}=D_i^{(q)}-A_i^{(q)}.
\]
This corruption removes part of the observed dependency structure while leaving the target memory unchanged, so retrieval requires identifying the correct underlying body-sensor organization from incomplete relational information.

We again compare Graph-DAM against two Euclidean baselines. First, Eu-DAM applies the classical Euclidean associative-memory update to $\operatorname{vec}(L_i)$. It therefore receives exactly the same graph representation and the same corrupted graph query as Graph-DAM, but treats each Laplacian as an ordinary vector. This comparison directly isolates the effect of using graph-operator geometry rather than ambient Euclidean geometry.

Second, Raw Eu-DAM operates directly on a conventional non-graph representation of the original $45$-channel sensor recording. For each sensor channel, we extract compact time-domain summaries including its mean, standard deviation, root-mean-square amplitude, median, temporal-difference magnitude, and amplitude range. For the raw query, $20\%$ of the temporal observations are removed, the missing observations are reconstructed by interpolation, and the same feature representation is recomputed. Thus, Raw Eu-DAM tests whether conventional vector-valued summaries of the distributed wearable measurements are sufficient for associative retrieval without explicitly representing the dependency structure among body-mounted sensors.

For all three methods, we use the same $N=30$ stored memories, $25$ logarithmically spaced inverse temperatures $\beta_n\in[10^{-1},10^3]$, five retrieval steps, and five independent corruption trials. Retrieval accuracy is the proportion of corrupted queries for which the corresponding stored memory is correctly recovered, and error bars report standard errors across the five trials.

\begin{figure}[!t]
    \centering
    \includegraphics[width=0.62\linewidth]{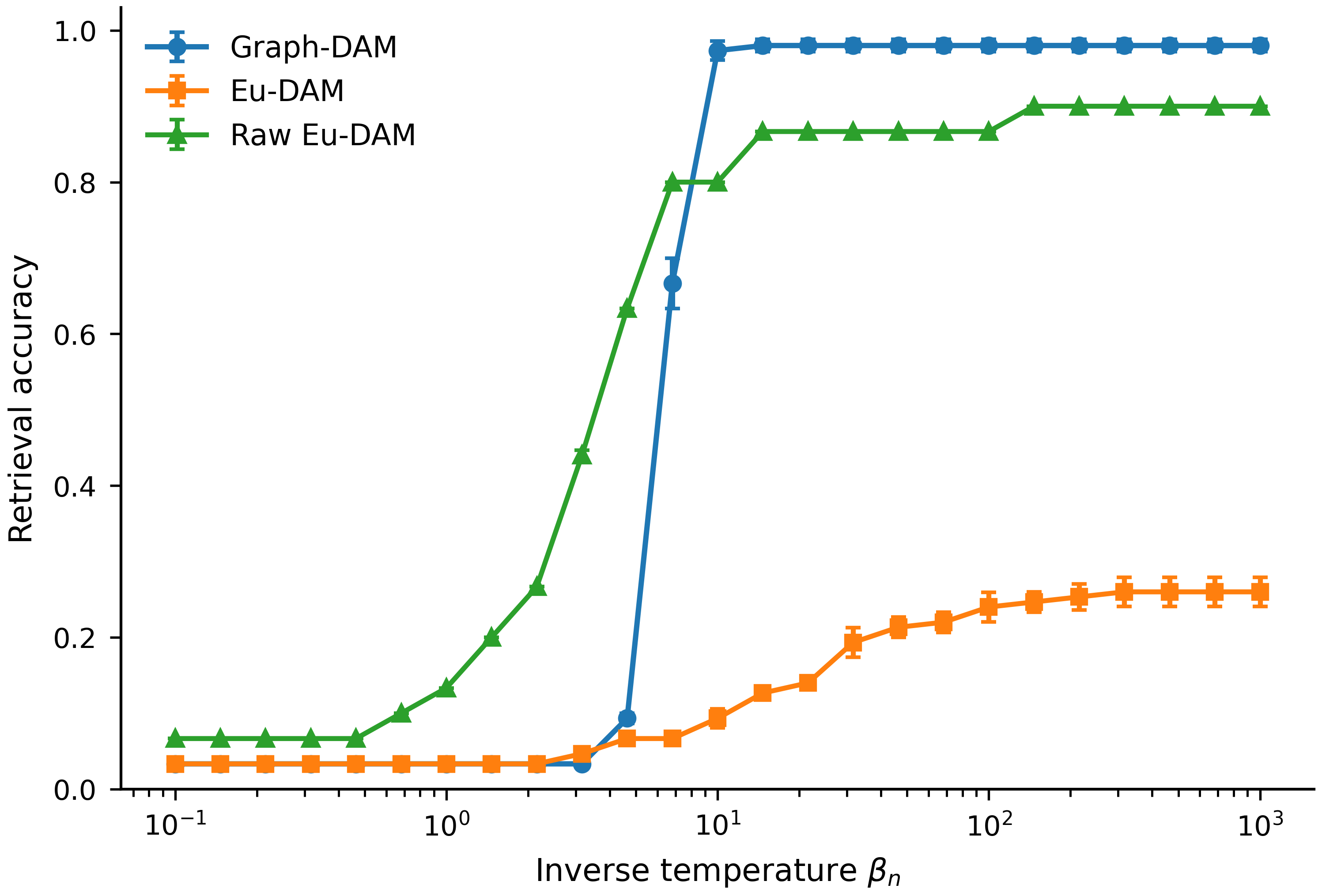}
    \caption{\textbf{UCI DSA retrieval performance.}
    Retrieval accuracy of Graph-DAM, Eu-DAM, and Raw Eu-DAM as a function of inverse temperature $\beta_n$. Graph-DAM rapidly approaches approximately $98\%$ retrieval accuracy and remains stable in the high-inverse-temperature regime, while Eu-DAM and Raw Eu-DAM plateau at approximately $26\%$ and $90\%$, respectively. Error bars denote standard errors over five independent corruption trials.}
    \label{fig:dsa_acc}
\end{figure}

\begin{figure}[!t]
    \centering
    \includegraphics[width=\linewidth]{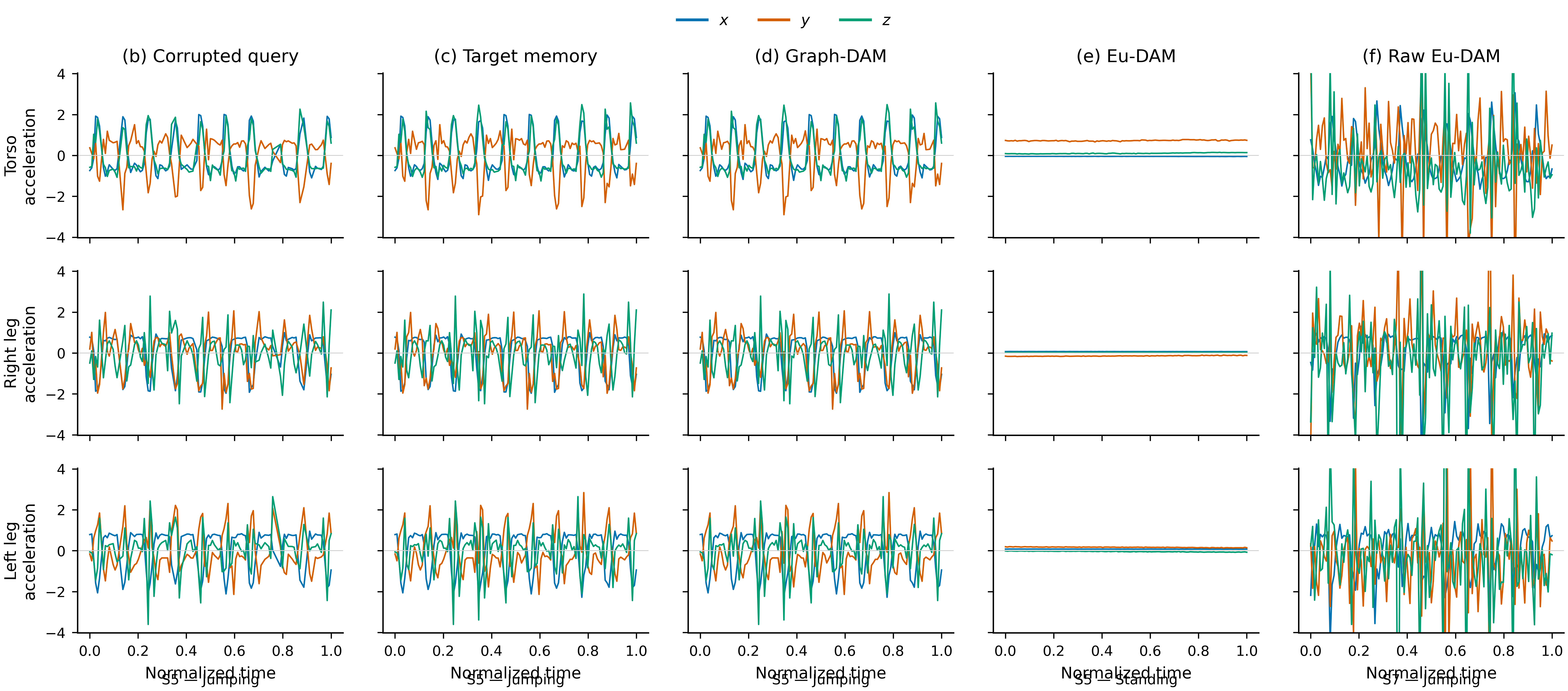}
    \caption{\textbf{UCI DSA retrieval example visualized in the original sensor space.}
    The rows show the three-axis acceleration signals measured at the torso, right leg, and left leg, respectively. From left to right, we show the corrupted query, its true target memory, and the memories retrieved by Graph-DAM, Eu-DAM, and Raw Eu-DAM. The target is a jumping recording from subject S5. Graph-DAM correctly retrieves the S5 jumping memory, whereas Eu-DAM retrieves an S5 standing memory and Raw Eu-DAM retrieves an S7 jumping memory. Signals are standardized using the target-memory statistics to allow direct comparison of their temporal patterns and amplitudes.}
    \label{fig:dsa_signals}
\end{figure}

\begin{figure}[!t]
    \centering
    \includegraphics[width=\linewidth]{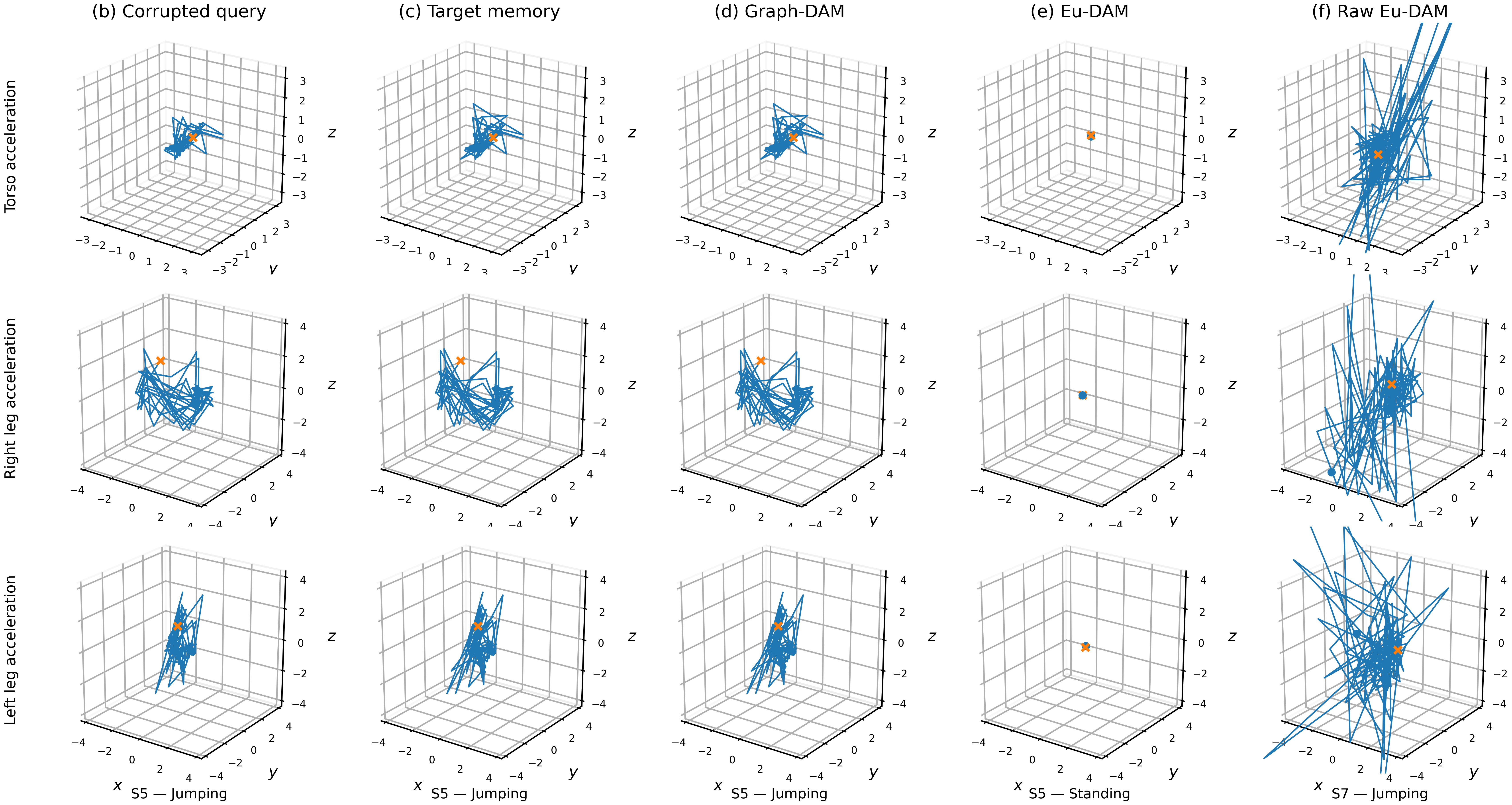}
    \caption{\textbf{Three-dimensional acceleration trajectories for the UCI DSA retrieval example.}
    The rows visualize the three-dimensional acceleration trajectories $(a_x,a_y,a_z)$ measured at the torso, right leg, and left leg, respectively, for the corrupted query, target memory, and memories retrieved by Graph-DAM, Eu-DAM, and Raw Eu-DAM. Graph-DAM recovers the true S5 jumping memory, Eu-DAM retrieves an S5 standing memory, and Raw Eu-DAM retrieves an S7 jumping memory. The trajectories provide a geometric visualization in the original sensor domain of the whole-body motion associated with each retrieved memory. Circles and crosses indicate the starting and ending points of each trajectory, respectively.}
    \label{fig:dsa_3d}
\end{figure}

\textbf{Main results.}
Figure \ref{fig:dsa_acc} shows that Graph-DAM achieves substantially higher retrieval accuracy than both Euclidean baselines in the high-inverse-temperature regime. At small $\beta_n$, all three methods remain close to chance-level retrieval. As $\beta_n$ increases, Raw Eu-DAM initially improves more rapidly, reflecting the substantial activity information already contained in conventional sensor features. Graph-DAM then undergoes a sharp transition and reaches approximately $98\%$ retrieval accuracy around $\beta_n=10$, remaining essentially constant thereafter. In contrast, Eu-DAM improves only gradually and plateaus at approximately $26\%$, while Raw Eu-DAM reaches approximately $90\%$.

The particularly large difference between Graph-DAM and Eu-DAM is informative because these two methods receive exactly the same graph-valued memories and corrupted graph queries. Their difference lies in how the Laplacians are compared and used in the associative-memory dynamics. Eu-DAM treats the entries of the Laplacian as coordinates in an ambient Euclidean space, whereas Graph-DAM uses the operator geometry induced by the graph Laplacians. The substantial performance gap therefore indicates that preserving this operator structure can be important for retrieving the relational patterns encoded by the wearable-sensor graphs.

The comparison with Raw Eu-DAM provides a complementary perspective. Raw Eu-DAM achieves considerably higher accuracy than Eu-DAM, demonstrating that the original wearable measurements contain strong information about activity identity even without an explicit graph representation. Nevertheless, its high-$\beta_n$ accuracy remains below that of Graph-DAM. The graph representation explicitly captures how sensor components at different body locations co-vary, thereby representing whole-body coordination rather than only summaries of individual sensor signals. The results in Figure \ref{fig:dsa_acc} suggest that this relational information provides additional discrimination among stored memories beyond that available from conventional raw-signal features.

Figures \ref{fig:dsa_signals} and \ref{fig:dsa_3d} provide a concrete example of this distinction. The target memory corresponds to a jumping recording from subject S5. Graph-DAM correctly retrieves the S5 jumping memory, whereas Eu-DAM retrieves an S5 standing memory and Raw Eu-DAM retrieves a jumping recording from subject S7. Figure \ref{fig:dsa_signals} visualizes the associated acceleration signals at the torso and both legs. The corrupted query retains the characteristic coordinated oscillatory motion of the S5 target, and the memory retrieved by Graph-DAM reproduces the corresponding temporal patterns across all three body locations. In contrast, the S5 standing memory selected by Eu-DAM contains little temporal variation and is visibly inconsistent with the dynamic structure of the jumping query.

The Raw Eu-DAM retrieval is more subtle. It correctly identifies a jumping recording, and therefore retains the large and rapidly varying accelerations characteristic of the activity, but it retrieves the memory of a different subject, S7. Its signals consequently differ substantially from the S5 target in their detailed amplitudes and temporal patterns. This example highlights the distinction between recognizing a coarse activity type and retrieving the exact stored memory. Conventional raw-signal features can capture strong activity-level similarity, but they do not necessarily preserve the finer subject- and realization-specific coordination structure required for exact associative retrieval.

The same behavior can be viewed geometrically in Figure \ref{fig:dsa_3d}. The three-dimensional acceleration trajectories of the corrupted query and the S5 target exhibit closely related structures at the torso and both legs, and Graph-DAM retrieves the corresponding S5 memory. The standing memory selected by Eu-DAM collapses to a small region of the three-dimensional sensor space, reflecting the absence of the dynamic motion characteristic of jumping. The S7 jumping memory selected by Raw Eu-DAM remains highly dynamic but exhibits a visibly different trajectory geometry from the S5 target. Together with the retrieval-accuracy results, these examples illustrate that Graph-DAM can preserve structured dependencies among spatially distributed sensor components and distinguish different realizations of similar physical activities. In contrast to the smartphone HAR experiment, where the graph structure is induced from temporal and cross-channel dependencies, the DSA experiment provides an especially direct physical interpretation of graph-valued memory: the graph represents coordinated motion across sensors distributed over the human body.

\bibliographystyle{abbrvnat}
\bibliography{example}

\end{document}